%% file: ms.tex
\def\year{2021}\relax
\documentclass[letterpaper]{article} 
\usepackage{aaai21}  
\usepackage{times}  
\usepackage{helvet} 
\usepackage{courier}  
\usepackage[hyphens]{url}  
\usepackage{graphicx} 
\usepackage{natbib}  
\usepackage{caption} 
\usepackage{subcaption}
\usepackage{amsmath,amssymb,amsfonts}
\usepackage{amsthm}
\usepackage{mathtools}
\usepackage{textcomp}
\usepackage{xcolor}

\usepackage{booktabs}
\usepackage[ruled,vlined, linesnumbered, boxed]{algorithm2e}
\usepackage{makecell}
\usepackage{bm}
\usepackage{thmtools}
\usepackage{pgfplots}
\usepgfplotslibrary{fillbetween}
\pgfplotsset{compat=1.16}
\pgfdeclarelayer{ft}
\pgfdeclarelayer{bg}
\pgfsetlayers{bg,main,ft}

\usepackage[switch]{lineno}

\newtheorem{theorem}{Theorem}

\newtheorem{example}[theorem]{Example}

\newcommand{\N}{\mathbb{N}}

\newcommand{\effortvec}{\vec{\beta}} 
\newcommand{\opteffortvec}{\vec{\beta}^\star} 
\newcommand{\opteffort}[1]{\beta_{#1}^\star} 
\newcommand{\effort}[1]{\beta_{#1}} 
\newcommand{\discretization}{\Psi} 
\newcommand{\discretizationlevel}[1]{\psi_{#1}} 

\newcommand{\featurevec}[1]{\vec{y}_{#1}} 
\newcommand{\Xit}{X_i^{(t)}} 
\newcommand{\discretizationgap}{\Delta} 
\newcommand{\maxregret}{R_{\text{max}}} 
\newcommand{\minregret}{R_{\text{min}}} 

\newcommand{\selfucb}{\textsc{selfUCB}} 
\newcommand{\ucb}{\textsc{UCB}} 

\newcommand{\reward}{\texttt{reward}} 

\title{Dual-Mandate Patrols: Multi-Armed Bandits for Green Security}
\author {
        Lily Xu,\textsuperscript{\rm 1}
        Elizabeth Bondi,\textsuperscript{\rm 1}
        Fei Fang,\textsuperscript{\rm 2}
        Andrew Perrault,\textsuperscript{\rm 1}
        Kai Wang,\textsuperscript{\rm 1}
        Milind Tambe\textsuperscript{\rm 1}\\
}
\affiliations {
    \textsuperscript{\rm 1} Harvard University \\
    \textsuperscript{\rm 2} Carnegie Mellon University \\
    \{lily\_xu, ebondi, aperrault, kaiwang\}@g.harvard.edu,
feif@cs.cmu.edu, 
milind\_tambe@harvard.edu
}

\begin{document}

\maketitle

\begin{abstract}
Conservation efforts in green security domains to protect wildlife and forests are constrained by the limited availability of defenders (i.e., patrollers), who must patrol vast areas to protect from attackers (e.g., poachers or illegal loggers). Defenders must choose how much time to spend in each region of the protected area, balancing exploration of infrequently visited regions and exploitation of known hotspots. We formulate the problem as a stochastic multi-armed bandit, where each action represents a patrol strategy, enabling us to guarantee the rate of convergence of the patrolling policy. However, a naive bandit approach would compromise short-term performance for long-term optimality, resulting in animals poached and forests destroyed. To speed up performance, we leverage smoothness in the reward function and decomposability of actions. We show a synergy between Lipschitz-continuity and decomposition as each aids the convergence of the other. In doing so, we bridge the gap between combinatorial and Lipschitz bandits, presenting a no-regret approach that tightens existing guarantees while optimizing for short-term performance. We demonstrate that our algorithm, LIZARD, improves performance on real-world poaching data from Cambodia.
\end{abstract}

\section{Introduction}



Green security efforts to protect wildlife, forests, and fisheries require defenders (patrollers) to conduct patrols across protected areas to guard against attacks (illegal activities) \cite{lober1992using}. For example, to prevent poaching, rangers conduct patrols to remove snares laid out to trap animals (Fig.~\ref{fig:rangers}). Green security games have been proposed to apply Stackelberg security games, a game-theoretic model of the repeated interaction between a defender and an attacker, to domains such as the prevention of illegal logging, poaching, or overfishing \cite{fang2015security}. 
A growing body of work develops scalable algorithms for Stackelberg games \cite{basilico2009leader,blum2014learning}.
Subsequent work has focused more on applicability to the real-world, employing machine learning to predict attack hotspots then using game-theoretic planning to design patrols 
\cite{nguyen2016capture,gholami2018adversary}. 

However, many protected areas lack adequate and unbiased past patrol data, disabling us from learning a reasonable adversary model in the first place \cite{moreto2015poaching}. As one of many examples, Bajo Madidi in the Bolivian Amazon was newly designated as a national park in 2019 \cite{franco2020rare}. The park is plagued with illegal logging, and patrollers do not have historical data from which to make predictions. The defenders do not want to spend patrol effort solely on information gathering; they must simultaneously maximize detection of attacks. As green security efforts get deployed on an ever-larger scale in hundreds of protected areas around the world \cite{xu2020stay}, addressing this information-gathering challenge becomes crucial.


\begin{figure}
  \centering
  \begin{subfigure}[t]{0.65\columnwidth}
  \centering
  \includegraphics[height=3.5cm]{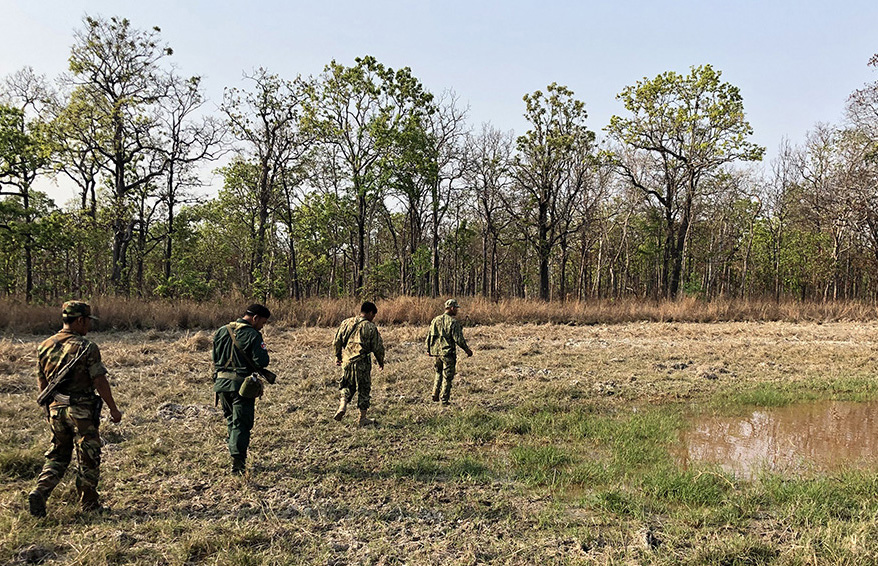}
  \end{subfigure}
  \begin{subfigure}[t]{0.32\columnwidth}
  \centering
  \includegraphics[height=3.5cm]{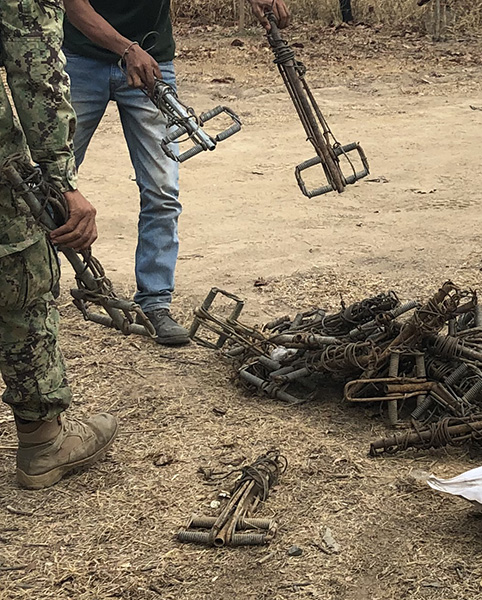}
  \end{subfigure}
  \caption{Rangers searching for snares (right) near a waterhole (left) in Srepok Wildlife Sanctuary in Cambodia. The waterhole is frequented by deer, pig, and bison, which are targeted by poachers.}
  \label{fig:rangers}
\end{figure}

Motivated by these practical needs, we focus on
conducting \textit{dual-mandate patrols}\footnote{Code available at \url{https://github.com/lily-x/dual-mandate}.}, with the goal of simultaneously detecting illegal activities and collecting valuable data to improve our predictive model, achieving higher long-term reward. 
The key challenge with dual-mandate patrols is the exploration--exploitation tradeoff: whether to follow the best patrol strategy indicated by historical data or conduct new patrols to get a better understanding of the attackers. Some recent work proposes using multi-armed bandits to formulate the problem \cite{xu2016playing,gholami2019dont}. Despite their advances, 
we show that these approaches require unrealistically long time horizons to achieve good performance. In the real world, these initial losses are less tolerable and can lead to stakeholders abandoning such patrol-assistance systems. As we are designing this system for future deployment, it is critical to account for these practical constraints. 
In this paper, we address real-world characteristics of green security domains to design dual-mandate patrols, prioritizing strong performance in the short term as well as long term. Concretely, we introduce LIZARD, a bandit algorithm that accounts for (i)~decomposability of the reward function, (ii)~
smoothness of the decomposed reward function across features, (iii)~monotonicity of rewards as patrollers exert more effort, and (iv)~availability of historical data. 
LIZARD leverages both decomposability and Lipschitz continuity simultaneously, \emph{bridging the gap between combinatorial and Lipschitz bandits}.
We prove that LIZARD achieves no-regret when adaptively discretizing the metric space, generalizing results from both \citet{chen2016combinatorial} and \citet{kleinberg2019bandits}. Specifically, we improve upon the regret bound of Lipschitz bandits for non-trivial cases with more than one dimension and extend combinatorial bandits to continuous spaces. 
Additionally, we show that LIZARD dramatically outperforms existing algorithms on real poaching data from Cambodia.



\section{Background}

Green security domains have been extensively modeled as green security games \cite{haskell2014robust,fang2016deploying,mc2016preventing,kamra2018policy}. In these games, resource-constrained defenders protect large areas (e.g., forests, savannas, wetlands) from adversaries who repeatedly attack them. Most work has focused on learning attacker behavior models from historical data \cite{nguyen2016capture,gholami2018adversary} and using these models to plan patrols with limited lookahead to prevent future attacks \cite{fang2015security,xu2017optimal}.

Despite successful deployment in some conservation areas \cite{xu2020stay}, researchers have recognized that it is not always possible to have abundant historical data 
when first deploying a green security algorithm. Fig.~\ref{fig:historical-data} demonstrates the shortcomings of a pure exploitation approach. Using a simulator built from real-world poaching data, we observe a large shortfall in reward unless an unrealistically large amount of historical data is collected. Naively exploiting historical data compromises reward unless a very large amount of data is collected. The gap between the orange and dashed green lines reveals the opportunity cost (regret) of acting solely based off historical data.

Some recent work has begun to consider the need to also explore. \citet{xu2016playing} follow an online learning paradigm, modeling the defender actions as pulling a set of arms (targets to patrol) at each timestep against an attacker who adversarially sets the rewards at each target. The model is then solved using FPL-UE, a variant of Follow the Perturbed Leader (FPL) algorithm \cite{kalai2005efficient}. \citet{gholami2019dont} propose a hybrid model, MINION, that employs the FPL algorithm to choose between two meta arms representing FPL-UE and a supervised learning model based on historical data.
However, both papers ignore domain structure such as feature similarity between targets and do not incorporate existing historical data into the online learners.

\begin{figure}
  \centering
  \input{figures/plot_historical_data.tex}
  \caption{Naively protecting 10 out of 100 potential poaching targets based on our predictions would require 6 years of data to accurately protect the most important 10 targets.} 
  \label{fig:historical-data}
\end{figure}
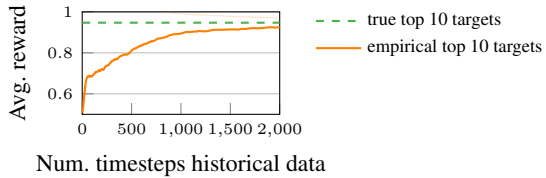



There is a plethora of work on multi-armed bandits (MABs) \cite{bubeck2012regret}.
For brevity, we focus on confidence bound--based algorithms; other common methods include epsilon-greedy and Thompson sampling \cite{lattimore2018bandit}. 
For stochastic MAB with finite arms, the seminal UCB1 algorithm \cite{auer2002finite} always chooses
the arm with the highest upper confidence bound and achieves no-regret, i.e., the expected reward is sublinearly close to always pulling the best arm in hindsight. For continuous arms, we must impose assumptions on the payoff function for tractability. Lipschitz continuity bounds the rate of change of a function~$\mu$, requiring that $|\mu(x) - \mu(y) | \leq L \cdot \mathcal{D}(x, y)$ for two points $x, y$ and distance function $\mathcal{D}$. We refer to $L$ as the Lipschitz constant. \citet{kleinberg2019bandits} propose the \emph{zooming algorithm} for Lipschitz MAB, which extends UCB1 with an adaptive refinement step to place more arms in high-reward areas of the continuous space. 
A parallel line of work applies Lipschitz-continuity assumptions to bandits to generate tighter confidence bounds \cite{bubeck2011x,magureanu2014lipschitz}. \citet{chen2016combinatorial} extend UCB1 to the combinatorial case with the \emph{CUCB algorithm}, which tracks individual payoffs separately and uses an oracle to select the best combination. However, CUCB does not extend to continuous arms. Our proposed algorithm, LIZARD, outperforms both the zooming and CUCB algorithms by explicitly handling arms that are both Lipschitz-continuous and combinatorial. 

Wildlife decline due to poaching demands fast action, so we cannot afford the infinite-time horizon usually considered by bandit algorithms. The case of short-horizon MABs is not well-studied. Some papers of this type \cite{gray2020regression} focus on adjusting exploration patterns rather than not address the types of challenges induced by smoothness and combinatorial rewards, as are present in the domain we consider.

\section{Problem Description}
\label{sec:problem-description}



The protected area for which we must plan patrols is discretized into $N$ targets, each associated with a feature vector $\featurevec{i} \in \mathbb{R}^{K}$ which is static across the time horizon~$T$. In the poaching prevention domain, for example, the $K$~features include geospatial characteristics such as distance to river, forest cover, animal density, and slope. We consider each timestep to be a short period of patrolling, e.g., a day, and a time horizon that reflects our foreseeable horizon for patrol planning, e.g., 1.5 years, corresponding to 500 timesteps. 


In each round, the defender determines an \emph{effort vector} $\effortvec = (\effort{1}, \ldots, \effort{N})$ which specifies the amount of effort to spend on each target. The defender has a budget~$B$ for total effort, i.e., $\sum_i \effort{i} \leq B$. In practice, $\effort{i}$ may represent the number of hours spent on foot patrolling in target~$i$. If the defender has two teams of patrollers, each of which can patrol for 5 hours a day, then $B=10$.
The planned patrols have to be conveyed clearly to the \emph{human} patrollers on the ground to then be executed \cite{plumptre2014efficiently}. Thus, an arbitrary value of $\effort{i}$ may be impractical. For example, we may ask the patrollers to patrol in an area for 30 minutes, but not 21.3634 minutes. Therefore, we enforce discretized patrol effort levels, requiring $\effort{i} \in \discretization=\{\discretizationlevel{1}, \ldots, \discretizationlevel{J}\}$ for $J$ levels of effort. 

Some targets may be attacked. The reward of a patrol corresponds to the total number of targets where attacks were detected \cite{critchlow2015spatiotemporal}. 
Let the expected reward of a patrol vector~$\effortvec$ be $\mu(\effortvec)$. Our objective is to specify an effort vector $\effortvec^{(t)}$ for each timestep~$t$ in an online fashion to minimize regret with respect to the optimal effort vector $\opteffortvec$ against a stochastic adversary, where regret is defined as $T \mu(\opteffortvec) - \sum_{t=1}^{T} \mu(\effortvec^{(t)})$. 

In practice, the likelihood of a defender detecting an attack is dependent on the amount of patrol effort exerted. In previous work on poaching prevention, a target can represent a large region, e.g., $1 \times 1$ km area, where snares are well-hidden. Thus, spending more time means the human patrollers can check the whole region more thoroughly and are more likely to detect snares. 
We represent the defender's expected reward at target~$i$ as a function $\mu_i(\effort{i}) \in [0, 1]$.
We define random variables $X^{(t)}_{i}$ as the observed reward (attack or no attack) from target~$i$ at time~$t$. Then $X^{(t)}_{i}$ follows a Bernoulli distribution with mean $\mu_i(\effort{i}^{(t)})$ with effort $\effort{i}^{(t)}$.

\subsection{Domain Characteristics}
\label{sec:characteristics}


We leverage the following four characteristics pervasive throughout green security domains to direct our approach. 

\paragraph{Decomposability} The overall expected reward for the defender is decomposable across targets and additive. 
For executing a patrol with effort~$\effortvec$ across all targets, the expected composite reward is a function $\mu(\effortvec) = \sum_{i=1}^{N} \mu_i(\effort{i})$.\paragraph{Lipschitz-continuity} As discussed, the expected reward to the defender at target $i$ is given by the function $\mu_i(\effort{i})$, which is dependent on effort $\effort{i}$. Furthermore, the expected reward depends on the features~$\featurevec{i}$ of that target, that is, $\mu_i(\effort{i}) = \widetilde{\mu}(\featurevec{i}, \effort{i})$ for all $i$. 
Past work to predict poaching patterns using machine learning models, which rely on assumptions of Lipschitz continuity, have been shown to perform well in real-world field experiments \cite{xu2020stay}. Thus, we assume that the reward function $\widetilde{\mu}(\cdot, \cdot)$ is Lipschitz-continuous in feature space as well as across effort levels.
That is, two distinct targets in the protected area with identical features will have identical reward functions, and two targets $a$ and $b$ with features $\featurevec{a}$, $\featurevec{b}$ and effort $\effort{a}$, $\effort{b}$ have rewards that differ by no more than
\begin{align}
| \widetilde{\mu}(\featurevec{a}, \effort{a}) - \widetilde{\mu}(\featurevec{b}, \effort{b}) | \leq L \cdot \mathcal{D}((\featurevec{a}, \effort{a}), (\featurevec{b}, \effort{b}))
\end{align}
for some Lipschitz constant~$L$ and distance function $\mathcal{D}$, such as Euclidean distance. Hence, the composite reward $\mu(\effortvec)$ is also Lipschitz-continuous. 

For simplicity of notation, we assume that the same Lipschitz constant applies over the entire space. In practice, we could achieve tighter bounds by separately estimating the Lipschitz constant for each dimension. 



\paragraph{Monotonicity} The more effort spent on a target, the higher the expected reward as our likelihood of finding a snare increases. That is, we assume $\mu(\effort{i})$ is monotonically non-decreasing in $\effort{i}$. Additionally, we assume that zero effort corresponds with zero reward ($\mu_i(0) = 0$), as defenders cannot prevent attacks on targets they do not visit. 

\paragraph{Historical data} Finally, many conservation areas have data from past patrols, which we use to warm start the online learning algorithm. To improve our short-term performance, we are careful in how we integrate historical data, as described in Sec.~$\ref{sec:historical-data}$. 



\subsection{Domain Challenges}

There are additional challenges in green security domains that prevent us from directly applying existing algorithms. No-regret guarantees describe the performance at infinite time horizons, but, practically, we require strong performance within short time frames. Losses in initial rounds are less tolerable, as the consequences of ineffective patrols are deforestation and loss of wildlife. In addition, conservation managers may lose faith in using an online learning approach. In response, we provide an algorithm with infinite horizon guarantees and also empirically show strong performance in the short term.

\section{LIZARD Online Learning Algorithm}\label{sec:algorithm}

We present our algorithm, LIZARD (LIpschitZ Arms with Reward Decomposability), for online learning in green security domains. 


Standard bandit algorithms suffer from the curse of dimensionality: the set of arms would be $\discretization^N$, which has size $J^N$. 
Thus, we cast the problem as a combinatorial bandit \cite{chen2016combinatorial}. At each iteration, we choose a patrol strategy $\effortvec$ that satisfies the budget constraint and observe the patrol outcome of each target~$i$ under the chosen effort~$\effort{i}$. An \emph{arm} is one effort level~$\effort{i}$ on a specific target~$i$; a \emph{super arm} is $\effortvec$, a collection of $N$ arms.
By tracking decomposed rewards, we only need to track observations from $NJ$ arms. 

We now maintain exponentially fewer samples, but the number of arms is still prohibitively large. With, say, $N=100$ targets and $J=10$ effort levels, we would need to develop estimates of reward for 1,000 arms. To address this challenge, we leverage feature similarity between arms to speed up learning, demonstrating the synergy between decomposability and Lipschitz-continuity. In this section, we show how we transfer knowledge between arms with similar effort levels and features.

\begin{algorithm}[t]
\caption{LIZARD}
\label{alg:decomposed}
\textbf{Inputs:} Number of targets~$N$, time horizon~$T$, budget~$B$, discretization levels~$\discretization$, target features~$\featurevec{i}$ 
\\
$n(i, \discretizationlevel{j}) = 0$, $\reward(i, \discretizationlevel{j}) = 0$ ~ $\forall i \in [N], j \in [J]$ \\
\For{$t = 1, 2, \ldots, T$} {
Compute $\ucb_t$ using Eq.~\ref{eqn:tighter_ucb}\\
Solve $\mathcal{P}(\ucb_t, B, N, T)$ to select super arm $\effortvec$ \\
Observe rewards $X_1^{(t)}, X_2^{(t)}, \ldots, X_n^{(t)}$ \\
\For{$i = 1, 2, \ldots, N$} {
$\reward(i, \effort{i}) = \reward(i, \effort{i}) + \Xit$ \\
$n(i, \effort{i}) = n(i, \effort{i}) + 1$
}
}
\end{algorithm}

\subsection{Upper Confidence Bounds with Similarity}

We take an upper confidence bound (UCB) approach where the rewards are tracked separately for different arms. 
We show that we can incorporate Lipschitz-continuity of the reward functions into the UCB of each arm to achieve tighter confidence bounds. 

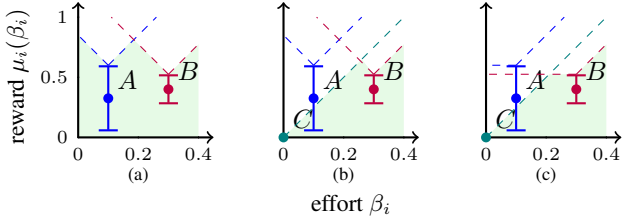
\begin{figure}[t]
    \begingroup
    \tikzset{every picture/.style={scale=0.4}}%
    \centering
    \begin{subfigure}[t]{0.36\columnwidth}
        \centering
        \input{figures/ucb_lipschitz.tex}
        \label{fig:ucb-lipschitz}
    \end{subfigure}
    \hfill
    \begin{subfigure}[t]{0.25\columnwidth}
        \centering
        \input{figures/ucb_zero.tex}
        \label{fig:ucb-zero}
    \end{subfigure}
    \hfill
    \begin{subfigure}[t]{0.25\columnwidth}
        \centering
        \input{figures/ucb_increasing.tex}
        \label{fig:ucb-increasing}
    \end{subfigure}
    \endgroup
    \caption{The Lipschitz assumption enables us to prune confidence bounds. We show the impact of each $\selfucb$s on the $\ucb$s of other arms in effort space of target~$i$. The solid brackets represent the $\selfucb$s. The dashed lines represent the bounds imposed by each arm on the rest of the space. The shaded green region covers the potential value of the reward function at different levels of effort. We visualize the additive effect of (a)~Lipschitz-continuity, (b)~zero effort yields zero reward, and (c)~monotonicity. Note that these plots demonstrate $\ucb$s for one target and that Lipschitz continuity also applies \textit{across} targets based on feature similarity. 
    }
    \label{fig:ucb}
\end{figure}
Let $\bar{\mu}_t(i, j) = \reward_t(i, \discretizationlevel{j}) / n_t(i, \discretizationlevel{j})$ be the average reward of target~$i$ at effort~$\discretizationlevel{j}$ given cumulative empirical reward $\reward_t(i, \discretizationlevel{j})$ over $n_t(i, \discretizationlevel{j})$ arm pulls by timestep~$t$. Let $r_t(i, j)$ be the \emph{confidence radius} at time~$t$ defined as 
\begin{small}
\begin{align}
    r_t(i, j) = \sqrt{\frac{3\log(t)}{2 n_t(i, \discretizationlevel{j})}} \ .
    \label{eq:conf}
\end{align}
\end{small}%
We distinguish between UCB and a term we call $\selfucb$. The $\selfucb$ of an arm $(i, j)$ representing target~$i$ with effort level~$j$ at time~$t$ is the UCB of an arm based only on its own observations, given by
\begin{align}
    \selfucb_t(i, j) = \bar{\mu}_t(i, j) + r_t(i, j) \ .
\end{align}
This definition of $\selfucb$ corresponds with the standard interpretation of confidence bounds from UCB1 \cite{auer2002finite}. The UCB of an arm is then computed by taking the minimum of the bounds of all $\selfucb$s as applied to the arm. These bounds are determined by adding the distance between arm $(i, j)$ and all other arms $(u, v)$ to the $\selfucb$:
\begin{align}\label{eqn:tighter_ucb}
\ucb_t(i, j) &= \min\limits_{\substack{u \in [N] \\ v \in [J]}} \{ \selfucb_t(u, v) + L \cdot \textit{dist} \} \\ 
\textit{dist} &= \max\{0,  \discretizationlevel{v} - \discretizationlevel{j} \} + \mathcal{D}(\featurevec{i}, \featurevec{u}) \nonumber
\end{align}
which exploits Lipschitz continuity between the arms. See Fig.~\ref{fig:ucb} for a visualization. The distance between two arms depends on the similarity of their features and effort. The first term of $\textit{dist}$ considers similarity of effort level (Fig.~\ref{fig:ucb}a); the second considers feature similarity between targets according to distance function $\mathcal{D}$. We define $\ucb_t(i, 0) = 0$ for all $i \in [N]$ due to the assumption that zero effort yields zero reward (Fig.~\ref{fig:ucb}b). To address the monotonically non-decreasing reward across effort space, we constrain the first term of 
$\textit{dist}$ to be nonnegative (Fig.~\ref{fig:ucb}c). 



\subsection{Super Arm Selection}
With the computed UCBs, the selection of super arms (patrol strategies) can be framed as a knapsack optimization problem. We aim to maximize the value of our sack (sum of the UCBs) subject to a budget constraint (total effort). To solve this problem, we use the following integer linear program~$\mathcal{P}$.
\begin{align*}
    \max_z \quad & \sum\nolimits_{i \in [N]} \sum\nolimits_{j \in [J]} z_{i, j} \cdot \ucb_t (i, j) \tag{$\mathcal{P}$} \\
    \text{s.t.} \quad 
    & z_{i, j} \in \{0, 1\} \qquad \qquad \qquad \forall i \in [N], j \in [J] \\
    & \sum\nolimits_{j \in [J]}  z_{i, j} = 1 \qquad \qquad \forall i \in [N] \\
    & \sum\nolimits_{i \in [N]} \sum\nolimits_{j \in [J]} z_{i, j} \discretizationlevel{j} \leq B
\end{align*}
There is one auxiliary variable $z_{i, j}$, constrained to be binary, for each level of patrol effort for each target. Setting $z_{i,j}=1$ means we exert $\discretizationlevel{j}$ effort on target~$i$. The second constraint sets $\effort{i}$ by requiring that we pull one arm per target (which may be the zero effort arm $\effort{i} = 0$). The third constraint mandates that we stay within budget. This integer program has $NJ$ variables and $N+1$ constraints. 
Pseudocode for the LIZARD algorithm is given in Algorithm~\ref{alg:decomposed}. 



\subsection{Incorporating Historical Data}
\label{sec:historical-data}
We incorporate historical data to further improve bounds without compromising the regret guarantee (see Sec.~\ref{sec:confidence_bounds}). \citet{shivaswamy2012multi} show that, in the infinite horizon case, a logarithmic amount of historical data can reduce regret from logarithmic to constant. However, care must be taken to consider short-term effects. Standard methods that include historical arm pulls as if they occurred online can result in poor short-term performance, as in the following example:
\begin{example}
Let $A$ and $B$ be two arms with noise-free rewards $\mu(A) = 1$ and $\mu(B) = 0$. Suppose that in the historical data, $A$ is pulled $H \gg 1$ times and $B$ has never been pulled. If those historical arm pulls are counted in the calculation of confidence radii, then we must pull the bad arm $B$ first $s \geq O(H / \log H)$ times until the UCB of arm $B$ is smaller than the UCB of arm $A$:
\begin{align*}
& \sqrt{\frac{3 \log{s}}{2s}} = 0 + r_s(B) \leq 1 + r_s(A) = 1 + \sqrt{\frac{3 \log{s}}{2H}}
\end{align*}
yielding linear regret $O(s) = O \left( \frac{H}{\log H} \right)$ in the first $s$~rounds.
\end{example}
Intuitively, the more times the optimal arm is pulled in the historical data, the greater the short-term regret. Because we expect the historical data to oversample good arms (as patrollers prefer visiting areas that are frequently attacked), we design LIZARD to avoid this problem by ignoring the number of historical arm pulls when computing the confidence radii. Specifically, LIZARD initializes the observed rewards $\reward_t(\cdot, \cdot)$ and the number of pulls $n_t(\cdot, \cdot)$ with the historical observations, but, when computing the confidence radii $r_t(i, j)$, LIZARD only considers the number of pulls from online play.




\section{Regret Analysis}


We provide a regret bound for Algorithm~\ref{alg:decomposed} with fixed discretization (Sec.~\ref{sec:fixed_discretization}), which is useful in practice but cannot achieve theoretical no-regret due to the discretization factor. 
Thus, we then offer Algorithm~\ref{alg:adaptive_decomposed} with adaptive discretization to achieve no-regret (Sec.~\ref{sec:adaptive_discretization}), showing that there is no barrier to achieving no regret in practice other than the need for fixed discretization in operationalizing our algorithm in the field. 
Our regret bound improves upon that of the zooming algorithm of \citet{kleinberg2019bandits} for all dimensions $d > 1$. In our problem formulation, each dimension of the continuous action space represents a target, so $d = N$. In fact, the regret bound for the zooming algorithm is a provable lower bound; we are able to improve this lower bound through decomposition (Sec.~\ref{sec:confidence_bounds}). Furthermore, we extend the line of research on combinatorial bandits, generalizing the CUCB algorithm from \citet{chen2016combinatorial} to continuous spaces.

Full proofs are included in the appendix.


\subsection{Fixed Discretization}\label{sec:fixed_discretization}


\begin{restatable}[]{theorem}{fixedDiscretization}\label{thm:fixed_discretization}
Given the minimum discretization gap $\discretizationgap$, number of arms~$N$, Lipschitz constant~$L$, and time horizon~$T$, the regret bound of Algorithm~\ref{alg:decomposed} with $\selfucb$ is
\begin{align}\label{eqn:fixed_regret_bound}
\resizebox{.85\hsize}{!}{%
$\text{Reg}_{\discretizationgap}(T) \leq O \left( N L\discretizationgap T + \sqrt{N^3 \discretizationgap^{-1} T \log{T}} + N^2 L\discretizationgap^{-1} \right)$
}.
\end{align}
\end{restatable}

\begin{proof}[Proof sketch]
The regret in Theorem~\ref{thm:fixed_discretization} comes from (i) discretization regret in the first term of Eq.~\ref{eqn:fixed_regret_bound} and (ii)~suboptimal arm selections in the last two terms.
The discretization regret is due to inaccurate approximation caused by discretization, where the error can be bounded by rounding the optimal arm selection and fractional effort levels to their closest discretized levels.
The suboptimal arm selections are due to insufficient samples across all discretized subarms 
and have sublinear regret in terms of $T$.
\end{proof}


Under a finite horizon, it is sufficient to pick a discretization level based on the time horizon. For example, given a finite horizon~$T$, we can pick a discretization that is optimal relative to $T$. We want to use a finer discretization to minimize the discretization regret, while at the same time minimizing the number of effort levels to explore.
Thus, we trade off between the selection regret and discretization regret: we want them to be of the same order as the regret bound in Theorem~\ref{thm:fixed_discretization}, i.e., $O (\sqrt{N^3\discretizationgap^{-1} T \log{T}}) = O(NL \discretizationgap T)$, or equivalently $\discretizationgap = (\log{T} / {T})^\frac{1}{3}N^\frac{1}{3} L^{-\frac{2}{3}}$.
This result matches the intuition that we should use a finer discretization when either (i)~the total timesteps is larger, (ii)~the number of targets is smaller, or (iii)~the Lipschitz constant is larger (i.e., function values change more drastically). 

With Theorem~\ref{thm:fixed_discretization} we observe that the barrier to achieving no-regret is the discretization error which is linear in all terms, which brings us to adaptive discretization. 



\subsection{Adaptive Discretization}
\label{sec:adaptive_discretization}


To achieve no-regret with an infinite time horizon, we need adaptive discretization. 
Adaptive discretization is less practical on the ground, but would be useful in other bandit settings where we could more precisely spend our budget such as influence maximization and facility location.
As shown in Algorithm~\ref{alg:adaptive_decomposed}, we begin with a coarse patrol strategy, beginning with binary decisions on whether or not to visit each target, then gradually progress to a finer discretization. 

\begin{algorithm}[t]
\caption{Adaptively Discretized LIZARD}
\label{alg:adaptive_decomposed}
\textbf{Inputs:} Number of targets $N$, time horizon~$T$, budget~$B$, target features~$\featurevec{i}$, Lipschitz constant $L$ \\
Discretization levels~$\discretization = \{0,1\}$ 
, gap $\discretizationgap = 1$ \\
$T_k = \frac{N 2^{3k}}{L^2} \log{\frac{N 2^{3k}}{L^2}} ~\forall k \in \N \cup \{ 0 \}$\\
$n(i, \psi_j) = 0$, $\reward(i, \psi_j) = 0$ $\forall i \in [N], j \in [J]$ \\
\For{$t = 1, 2, \ldots, T$} {
\If{$t > \sum\nolimits_{j=0}^{k-1} T_j$}{Set $\discretizationgap=2^{-k}$ and $\discretization = \{0, \discretizationgap, ..., 1\}$}
Compute $\ucb_t$ using Eq.~\ref{eqn:tighter_ucb} \\
Solve $\mathcal{P}(\ucb_t, B, N, T)$ to select super arm $\effortvec$ \\
Observe rewards $X_1^{(t)}, X_2^{(t)}, \ldots, X_n^{(t)}$ \\
\For{$i = 1, 2, \ldots, N$} {
$\reward(i, \effort{i}) = \reward(i, \effort{i}) + \Xit$ \\
$n(i, \effort{i}) = n(i, \effort{i}) + 1$
}
}
\end{algorithm}

\begin{restatable}[]{theorem}{adaptiveDiscretization}\label{thm:adaptive_discretization}
Given the number of arms~$N$, Lipschitz constant~$L$, and time horizon~$T$, the regret bound of Algorithm~\ref{alg:adaptive_decomposed} with $\selfucb$ is be given by
\begin{align}\label{eqn:adaptive_regret_bound}
    \text{Reg}(T) & \leq O \left( L^{\frac{4}{3}} N T^{\frac{2}{3}} (\log{T})^{\frac{1}{3}} \right) \ .
\end{align}
\end{restatable}
\begin{proof}[Proof sketch]
The regret bound in Theorem~\ref{thm:fixed_discretization} is not sublinear due to the additional discretization error. An intuitive way to alleviate this error is to adaptively reduce the discretization gap. We run each discretization gap $\discretizationgap$ for $T_\discretizationgap = \frac{N}{L^2 \discretizationgap^3} \log{\frac{N}{L^2 \discretizationgap^3}}$ time steps to make the discretization error and the selection error of the same order. We then start over with a finer discretization $\discretizationgap/2$ to make the discretization error smaller. After summing the regret from all different phrases of discretization, we achieve sublinear regret as shown in Eq.~\ref{eqn:adaptive_regret_bound}.
\end{proof}

Under reasonable problem settings, $T$ dominates all other variables, so the regret in Theorem~\ref{thm:adaptive_discretization} is effectively of order $O(T^{\frac{2}{3}} (\log{T})^{\frac{1}{3}})$.
Our regret bound matches the bound of the zooming algorithm with covering dimension $d=1$.
Our setting is instead $N$-dimensional, falling into a space with covering dimension $d=N$. The regret bound for a metric space with covering dimension~$d$ is $O (T^{\frac{d+1}{d+2}} (\log{T})^{\frac{1}{d+2}})$, which approaches $\tilde{O} (T)$ as $d$ approaches infinity \cite{kleinberg2008multi}. More generally, our regret bound improves upon that of the zooming algorithm for any $d > 1$. 

Theorem~\ref{thm:adaptive_discretization} signifies that LIZARD can successfully decouple the $N$-dimensional metric space into individual sub-dimensions while maintaining the smaller regret order, showcasing the power of decomposibility.




\subsection{Tightening Confidence Bounds}
\label{sec:confidence_bounds}
We have so far offered regret bounds to account for decomposability and Lipschitz-continuity across effort space. We now guarantee that the regret bounds continue to hold with Lipschitz-continuity in feature space, monotonicity, and historical information. 
We first look at how prior knowledge affects the regret bound in the combinatorial bandit setting:
\begin{restatable}[]{theorem}{lipschitzCUCB}\label{thm:lipschitz_cucb}
Consider a combinatorial multi-arm bandit problem. If the bounded smoothness function given is $f(x) = \gamma x^\omega$ for some $\gamma > 0, \omega \in (0,1]$ and the Lipschitz upper confidence bound is applied to all $m$ base arms, the cumulative regret at time~$T$ is bounded by
\begin{align*}
\resizebox{.95\hsize}{!}{%
$\text{Reg}(T) \leq \frac{2 \gamma}{2 - \omega} (6 m \log{T})^{\frac{\omega}{2}} \cdot T^{1 - \frac{\omega}{2}} + \left(\frac{\pi^2}{3} + 1 \right) m \maxregret$
}
\end{align*}
where $\maxregret$ is the maximum regret achievable.
\end{restatable}

Theorem~\ref{thm:lipschitz_cucb} matches the regret of the CUCB algorithm \cite{chen2016combinatorial}, generalizing combinatorial bandits to continuous spaces.
Theorem~\ref{thm:lipschitz_cucb} also allows us to generalize Theorems \ref{thm:fixed_discretization} and~\ref{thm:adaptive_discretization} to our setting with a tighter $\textsc{UCB}$ from Lipschitz-continuity, which yields Theorem~\ref{thm:lipschitz_regret_bound}:

\begin{restatable}[]{theorem}{lipschitzRegretBound}\label{thm:lipschitz_regret_bound}
Given the minimum discretization gap $\discretizationgap$, number of targets $N$, Lipschitz constant $L$, and time horizon $T$, the regret bound of Algorithm~\ref{alg:decomposed} with $\ucb$ is
\begin{align}\label{eqn:lipschitz_fixed_regret_bound}
\resizebox{.9\hsize}{!}{%
$\text{Reg}_{\discretizationgap}(T) \leq O \left( N L\discretizationgap T + \sqrt{N^3 \discretizationgap^{-1} T \log{T}} + N^2 L\discretizationgap^{-1} \right)$
}
\end{align}
and the regret bound of Algorithm~\ref{alg:adaptive_decomposed} with $\ucb$ is
\begin{align}\label{eqn:lipschitz_adaptive_regret_bound}
    \text{Reg}(T) \leq O\left( L^{\frac{4}{3}} N T^{\frac{2}{3}} (\log{T})^{\frac{1}{3}} \right) \ .
\end{align}
\end{restatable}
Finally, when historical data is used, we can treat all the historical arm pulls as previous arm pulls with regret bounded by the maximum regret $\maxregret$.\footnote{Note that LIZARD treats all historical pulls as a single arm pull with return equal to the average return to avoid short-term losses (Sec.~\ref{sec:historical-data}). However, the regret bound covers the broader case of using \emph{any} finite number of historical arm pulls.}
This yields a regret bound with a time-independent constant and is thus still sublinear, achieving no-regret. Taken together, these capture all the properties of the LIZARD algorithm, so we can state:
\begin{restatable}[]{corollary}{lipschitzRegret}
The regret bounds of Algorithm~\ref{alg:decomposed} and Algorithm~\ref{alg:adaptive_decomposed} still hold with the inclusion of decomposability, Lipschitz-continuity, monotonicity, and historical data.
\end{restatable}

Theorem~\ref{thm:lipschitz_regret_bound} highlights the interplay between Lipschitz-continuity and decomposition.
The zooming algorithm achieves the provable lower bound on regret  $\tilde{O}(T^{\frac{d+1}{d+2}})$ \cite{kleinberg2004nearly}. \emph{The regret that we achieve improves upon that lower bound for all $d > 1$, which is only possible due to the addition of decomposition.} 






\section{Experiments}
\label{sec:experiments}

\begin{figure}
  \centering
  \begin{subfigure}[t]{.26\columnwidth}
  \centering
  \includegraphics[width=\columnwidth]{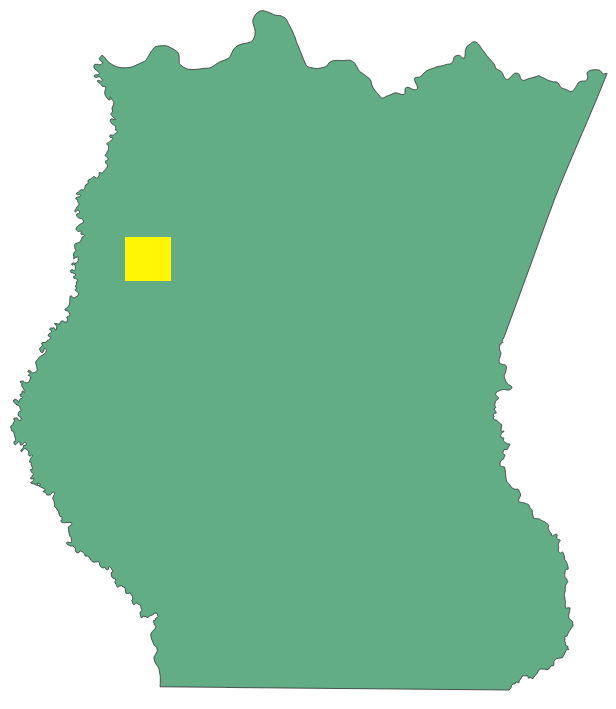}
  \end{subfigure}
  ~
  \begin{subfigure}[t]{.65\columnwidth}
  \centering
  \input{figures/real-world_reward_plot.tex}
  \end{subfigure} 
  \caption{Map of Srepok with a $5 \times 5$ km region highlighted and the real-world reward functions of the corresponding 25 targets.}
  \label{fig:sws-reward}
\end{figure}

\begin{table*}[t]
\centering
\caption{Performance across multiple problem settings, throughout which LIZARD achieves closest-to-optimal performance.}
\label{table:vary-parameters}
\setlength\tabcolsep{3.5pt} 
\resizebox{\textwidth}{!}{%
\begin{tabular}{r rrrrrrrr r rrrrrrrr}
\toprule
& \multicolumn{8}{c}{\textsc{Synthetic Data}} & \hspace{.3em} & \multicolumn{8}{c}{\textsc{Real-World Data}} \\
\cmidrule(lr){2-9} \cmidrule(lr){11-18}
\small{$(N, B) =$} & \multicolumn{2}{c}{\small{(25, 1)}} & \multicolumn{2}{c}{\small{(25, 5)}} & \multicolumn{2}{c}{\small{(100, 5)}} & \multicolumn{2}{c}{\small{(100, 10)}} && \multicolumn{2}{c}{\small{(25, 1)}} & \multicolumn{2}{c}{\small{(25, 5)}} & \multicolumn{2}{c}{\small{(100, 5)}} & \multicolumn{2}{c}{\small{(100, 10)}} \\
\cmidrule(lr){2-3} \cmidrule(lr){4-5} \cmidrule(lr){6-7} \cmidrule(lr){8-9} \cmidrule(lr){11-12} \cmidrule(lr){13-14} \cmidrule(lr){15-16} \cmidrule(lr){17-18}

\small{$T =$} & 
\small{200} & \small{500} & \small{200} & \small{500} & \small{200} & \small{500} & \small{200} & \small{500} &&
\small{200} & \small{500} & \small{200} & \small{500} & \small{200} & \small{500} & \small{200} & \small{500} \\ 

\midrule

\small{\textbf{LIZARD}} 
& $\bm{40.9}$ & $\bm{55.0}$ 
& $\bm{20.3}$ & $\bm{37.8}$ 
& $\bm{26.5}$ & $\bm{44.9}$ 
& $\bm{15.8}$ & $\bm{54.6}$ &

& $\bm{79.5}$ & $\bm{93.1}$ 
& $\bm{61.4}$ & $\bm{70.6}$ 
& $\bm{74.7}$ & $\bm{82.7}$ 
& $\bm{56.8}$ & $\bm{71.9}$ \\

\small{\textbf{CUCB}} 
& $5.7$ & $0.8$ 
& $15.7$ & $29.4$ 
& $-3.4$ & $-2.3$ 
& $-0.4$ & $12.1$ &

& $45.0$ & $57.1$ 
& $49.5$ & $68.3$ 
& $25.5$ & $50.9$ 
& $35.5$ & $49.4$ \\

\small{\textbf{MINION}} 
& $-42.7$ & $-46.4$ 
& $-60.0$ & $-35.9$ 
& $-1.1$ & $-53.3$ 
& $-9.1$ & $-21.3$ &

& $-13.1$ & $-39.3$ 
& $-28.7$ & $-10.1$ 
& $-18.4$ & $-14.4$ 
& $-18.6$ & $-13.7$ \\

\small{\textbf{Zooming}} 
& $-20.6$ & $-18.8$ 
& $-0.4$ & $-6.8$ 
& $-25.1$ & $-24.5$ 
& $-22.8$ & $-21.5$ &

& $-14.4$ & $-5.5$ 
& $40.5$ & $40.8$ 
& $-38.5$ & $-36.0$ 
& $-21.5$ & $-22.7$ \\
\bottomrule
\end{tabular}%
}
\end{table*}


We conduct experiments using both synthetic data and real poaching data. 
Our results validate that the addition of decomposition and Lipschitz-continuity not only improves our theoretical guarantees but also leads to stronger empirical performance. 
We show that LIZARD (Algorithm~\ref{alg:decomposed}) learns effectively within practical time horizons.

We consider a patrol planning problem with $N = 25$ or $100$ targets (each a 1~sq.~km grid cell), representing the region reachable from a single patrol post, and time horizon $T = 500$ representing a year and a half of patrols. 
The budget is the number of teams, e.g., $B=5$ corresponds to 5 teams of rangers. We use 50 timesteps of historical data, approximately two months of patrol, as we focus on achieving strong performance in parks with limited historical patrols.

\paragraph{Real-world data} We leverage patrol data from Srepok Wildlife Sanctuary in Cambodia to learn the reward function, as dependent on features and effort \cite{xu2020stay}. We train a machine learning classifier to predict whether rangers will detect poaching activity at each target for a given effort level. We then generate piecewise-linear approximations of the reward functions, shown in Fig.~\ref{fig:sws-reward}, which naturally intersects the origin and is nondecreasing with increased effort. 

\paragraph{Synthetic data} To produce synthetic data, we generate piecewise-linear functions that mimic the behavior of real-world reward functions. We define feature similarity as the maximum difference between the reward functions across all effort levels. 
We generate historical data with a biased distribution over the targets, corresponding to the realistic scenario where rangers spend more effort on targets that are easily accessible, such as those closest to a patrol post.

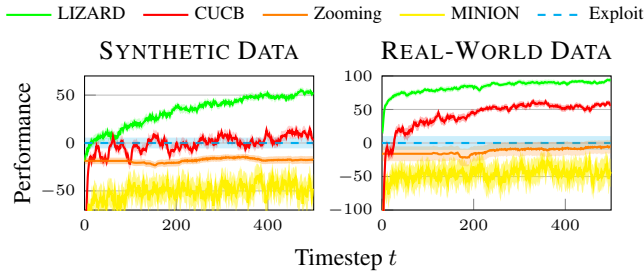
\begin{figure}
  \centering
  \input{figures/plot_performance.tex}
  \caption{Performance, measured in terms of percentage of reward achieved between $\textsc{optimal} - \textsc{exploit}$, over time. Shaded region shows standard error. Setting shown is $N = 25$, $B=1$. LIZARD (green) performs best.} 
  \label{fig:performance}
\end{figure}

\paragraph{Algorithms}
We compare to three baselines: \emph{CUCB} \cite{chen2016combinatorial}, \emph{zooming} \cite{kleinberg2019bandits}, and \emph{MINION} \cite{gholami2018adversary}. Zooming is an online learning algorithm that ignores decomposability, whereas CUCB uses decomposition but ignores similarity between arms. We use \emph{exploit history} as a naive baseline, which greedily exploits historical data with a static strategy. We compute the \emph{optimal} strategy exactly by solving a mixed-integer program over the true piecewise-linear reward functions, subject to the budget constraint. 
Experiments were run on a macOS machine with 2.4 GHz quad-core i5 processors and 16 GB memory. We take the mean of 30 trials.

Fig.~\ref{fig:performance} shows performance on both real-world and synthetic data, evaluated as the reward achieved at timestep~$t$, where the reward of historical exploit is 0 and of optimal is 1. The performance of LIZARD significantly surpasses that of the baselines throughout. 
On the real-world data, LIZARD provides a particular advantage over CUCB in the early rounds. 
Table~\ref{table:vary-parameters} shows the performance of each algorithm at $T=200$ and $500$, with varying values of $N$ and $B$. LIZARD beats the baselines in every scenario.

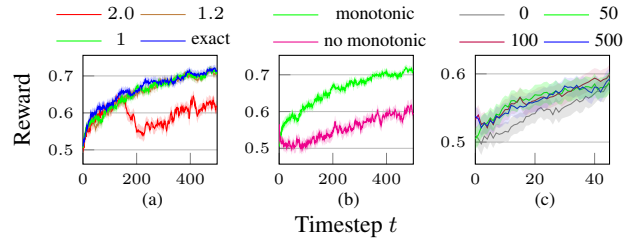
\begin{figure}
  \centering
  \input{figures/plot_vary_params.tex}
  \caption{Impact of (a)~different Lipschitz constants $L_i$, (b)~monotonicity assumption, (c)~varying amount of historical data. The green line in each plot depicts the setting used by LIZARD. Run on synthetic data with $N=25$, $B=1$.}
  \label{fig:vary_params}
\end{figure}

We return to the four characteristics of green security domains 
(Sec.~\ref{sec:characteristics})
and show that integrating each feature improves LIZARD. See Fig.~\ref{fig:vary_params} for results from our ablation study.
\textbf{Decomposability:} CUCB, a naive decomposed bandits algorithm, greatly exceeds the non-decomposed zooming algorithm throughout most of Table~\ref{table:vary-parameters}, demonstrating the value of decomposition. 
\textbf{Lipschitz-continuity:} Fig.~\ref{fig:vary_params}a reveals the value of information gained from knowing the exact value of the Lipschitz constant in each dimension ($L_i$ exact). We do not assume perfect knowledge of the Lipschitz constants $L_i$, instead selecting an approximate value $L_i = 1$ and using that same estimate across all dimensions. 
As shown, significantly overestimating $L_i$ with $L_i = 2$ hinders performance; this setting is closer to that of CUCB, with no Lipschitz continuity. 
\textbf{Monotonicity:} Fig.~\ref{fig:vary_params}b shows that the monotonicity assumption adds a significant boost to performance. \textbf{Historical data:} Fig.~\ref{fig:vary_params}c demonstrates the value of adding historical information in early rounds. By timestep 40, the benefit of historical data is negligible. 




\section{Conclusion}

We present LIZARD, an integrated algorithm for online learning in green security domains that leverages the advantages of decomposition and Lipschitz-continuity. With this approach, we transcend the proven lower regret bound of the zooming algorithm for Lipschitz bandits and extend combinatorial bandits to continuous spaces, showcasing their combined benefit. These results validate our approach of treating real-world conditions not as constraints but rather as useful features that lead to faster convergence. On top of achieving theoretical no-regret, we also demonstrate improved short-term performance empirically, increasing the usefulness of this approach in practice---particularly in high-stakes environments where we cannot compromise short-term reward. 

The next step is to move toward deployment. Researchers have demonstrated success deploying patrol strategies in the field to prevent poaching \cite{xu2020stay}. Our LIZARD algorithm could be directly implemented as-is by integrating historical data, estimating the Lipschitz constant, and setting the budget as the number of hours available per day. 

\section*{Ethics and Broader Impact}



This research has been conducted in close collaboration with domain experts to design our approach. Their insights have shaped the trajectory of our project. For example, our initial idea was to plan information-gathering patrols, taking an active learning approach to gather data where the predictive model was most uncertain. However, conservation experts pointed out in subsequent discussions that patrollers could not afford to spend time purely gathering data; they must prioritize preventing illegal poaching, logging, and fishing in the short-term. These priorities guided our project, particularly our focus on minimizing short-term regret as well as long-term regret. 

Our work is intended to serve as an assistive technology, helping rangers identify potentially snare-laden regions they otherwise might have missed, given that these parks can be up to thousands of square kilometers and rangers can only patrol a small number of regions at each timestep. We do not  intend this work to replace the critical role that rangers play in conservation management; no AI tool could substitute for the complex skills and domain insights that rangers and park managers have to plan and conduct patrols. 

A concern that might be raised could be that poachers who get access to this algorithm could predict where rangers might go and therefore place their snares more strategically. However, this algorithm would only be of use with access to the patrol data, which poachers would not have.

\bibliography{long,ref}

\include{appendix}

\end{document}

%% file: figures/plot_historical_data.tex
\begin{tikzpicture}
\pgfplotsset{
  width=0.5\linewidth,
  height=0.35\linewidth,
  xtick pos=left,
  ytick pos=left,
  tick label style={font=\tiny},
  ymajorgrids=true,
}

\begin{axis}[ 
  at={(0, 0)},
  xlabel={\footnotesize{Num. timesteps historical data}},
  ylabel absolute, ylabel style={ yshift=-.4cm},
  ylabel={\footnotesize{Avg. reward}},
  xlabel near ticks,
  legend cell align={left},
  legend style={
    align=left,
    at={(2.4, 1.1)},
    legend columns=1,
    font=\scriptsize, 
    draw=none,
    fill=none,
    /tikz/every even column/.append style={column sep=0.15cm}},
  ymin=.5,
  ymax=1,
  xmin=0,
  xmax=2000,
]

\addplot [dashed, green!40!gray, thick, mark=none] table [x=t, y=avg_top_true, col sep=comma] {data/historical_data_N100_top10.csv}; \addlegendentry{true top 10 targets}
\addplot [orange, thick, mark=none] table [x=t, y=avg_true, col sep=comma] {data/historical_data_N100_top10.csv}; \addlegendentry{empirical top 10 targets}

\addplot [draw=none, mark=none, name path=empirical_high] table [x=t, y expr=\thisrowno{1}+\thisrowno{2}, col sep=comma] {data/historical_data_N100_top10.csv};
\addplot [draw=none, mark=none, name path=empirical_low] table [x=t, y expr=\thisrowno{1}-\thisrowno{2}, col sep=comma] {data/historical_data_N100_top10.csv};
\tikzfillbetween[of=empirical_low and empirical_high, on layer=main]{orange, opacity=0.2};

\end{axis}

\end{tikzpicture}

%% file: figures/ucb_lipschitz.tex
\begin{tikzpicture}
[every label/.append style={text=red, font=\tiny}]








    
\fill[fill=green!80!black,fill opacity=0.1] (0, 0) -- (4, 0) -- (4, 3.1) -- (3, 2.1)  -- (1.825, 3.275) -- (1, 2.4) -- (0, 3.4) -- cycle;

\draw[thick,->] (0, 0) -- (4.5, 0) node[anchor=west] {}; 
\draw[thick,->] (0, 0) -- (0, 4.5) node[rotate=90, anchor=south] {}; 

\node[rotate=90, anchor=east] at (-2, 4.6) {\small reward $\mu_i(\beta_i)$};
\node[anchor=north] at (2.25, -1.6) {\phantom{\small effort $\beta_i$}};

\foreach \x in {0, 0.2, 0.4}
    \draw (10*\x cm,1pt) -- (10*\x cm,-1pt) node[anchor=north] {\scriptsize $\x$};
\foreach \y in {0, 0.5, 1}
    \draw (1pt,4*\y cm) -- (-1pt,4*\y cm) node[anchor=east] {\scriptsize $\y$};
    

\node[blue,fill,circle,scale=0.4] at (1, 1.3) {};
\node[anchor=south west] at (1, 1.3) {$A$};
\draw[thick,|-|,blue] (1, 0.2) -- (1, 2.4);
\draw[thin,dashed,blue] (1, 2.4) -- (0, 3.4);
\draw[thin,dashed,blue] (1, 2.4) -- (2.6, 4);

\node[purple,fill,circle,scale=0.4] at (3, 1.6) {};
\node[anchor=south west] at (3, 1.6) {$B$};
\draw[thick,|-|,purple] (3, 1.1) -- (3, 2.1);
\draw[thin,dashed,purple] (3, 2.1) -- (4, 3.1);
\draw[thin,dashed,purple] (3, 2.1) -- (1.1, 4);



\node[anchor=north] (a1) at (2, -.7){\scriptsize{(a)}};


\end{tikzpicture}

%% file: figures/ucb_zero.tex
\begin{tikzpicture}







    
\fill[fill=green!80!black,fill opacity=0.1] (0, 0) -- (4, 0) -- (4, 3.1) -- (3, 2.1)  -- (2.6, 2.6) -- cycle;

\draw[thick,->] (0, 0) -- (4.5, 0) node[anchor=west] {}; 
\draw[thick,->] (0, 0) -- (0, 4.5) node[anchor=south] {}; 

\node[anchor=north] at (2.25, -1.6) {\small effort $\beta_i$};

\foreach \x in {0, 0.2, 0.4}
    \draw (10*\x cm,1pt) -- (10*\x cm,-1pt) node[anchor=north] {\scriptsize $\x$};
    

\node[teal,fill,circle,scale=0.4] at (0, 0) {};
\node[anchor=south west] at (0, 0) {$C$};
\draw[thin,-,dashed,teal] (0, 0) -- (4, 4);

\node[blue,fill,circle,scale=0.4] at (1, 1.3) {};
\node[anchor=south west] at (1, 1.3) {$A$};
\draw[thick,|-|,blue] (1, 0.2) -- (1, 2.4);
\draw[thin,dashed,blue] (1, 2.4) -- (0, 3.4);
\draw[thin,dashed,blue] (1, 2.4) -- (2.6, 4);

\node[purple,fill,circle,scale=0.4] at (3, 1.6) {};
\node[anchor=south west] at (3, 1.6) {$B$};
\draw[thick,|-|,purple] (3, 1.1) -- (3, 2.1);
\draw[thin,dashed,purple] (3, 2.1) -- (4, 3.1);
\draw[thin,dashed,purple] (3, 2.1) -- (1.1, 4);



\node[anchor=north] (a1) at (2, -.7){\scriptsize{(b)}};


\end{tikzpicture}

%% file: figures/ucb_increasing.tex
\begin{tikzpicture}






    
\fill[fill=green!80!black,fill opacity=0.1] (0, 0) -- (4, 0) -- (4, 3.1) -- (3, 2.1) -- (2.1, 2.1) -- cycle;

\draw[thick,->] (0, 0) -- (4.5, 0) node[anchor=south] {}; 
\draw[thick,->] (0, 0) -- (0, 4.5) node[anchor=south] {}; 

\foreach \x in {0, 0.2, 0.4}
    \draw (10*\x cm,1pt) -- (10*\x cm,-1pt) node[anchor=north] {\scriptsize $\x$};
    

\node[teal,fill,circle,scale=0.4] at (0, 0) {};
\node[anchor=south west] at (0, 0) {$C$};
\draw[thin,-,dashed,teal] (0, 0) -- (4, 4);

\node[anchor=north] at (2.25, -1.6) {\phantom{\small effort $\beta_i$}};

\node[blue,fill,circle,scale=0.4] at (1, 1.3) {};
\node[anchor=south west] at (1, 1.3) {$A$};
\draw[thick,|-|,blue] (1, 0.2) -- (1, 2.4);
\draw[thin,dashed,blue] (1, 2.4) -- (0, 2.4);
\draw[thin,dashed,blue] (1, 2.4) -- (2.6, 4);

\node[purple,fill,circle,scale=0.4] at (3, 1.6) {};
\node[anchor=south west] at (3, 1.6) {$B$};
\draw[thick,|-|,purple] (3, 1.1) -- (3, 2.1);
\draw[thin,dashed,purple] (3, 2.1) -- (4, 3.1);
\draw[thin,dashed,purple] (3, 2.1) -- (0, 2.1);



\node[anchor=north] (a1) at (2, -.7){\scriptsize{(c)}};

\end{tikzpicture}

%% file: figures/real-world_reward_plot.tex
\begin{tikzpicture}
\pgfplotsset{
  width=\linewidth,
  height=0.6\linewidth,
  xtick pos=left,
  ytick pos=left,
  tick label style={font=\scriptsize},
  xmin=0,
  xmax=1,
  ymin=0,
}

\begin{axis}[
  xlabel={\scriptsize {Patrol effort $\beta_i$}},
  ylabel={\scriptsize {Avg reward $\mu_i(\beta_i)$}},
  xlabel near ticks,
  ylabel near ticks,
  xtick={0,.5,1},
]
\foreach \y in {0, ..., 24}{
    \addplot [blue, mark=none, thick, draw opacity=.3] table [x=x, y=\y, col sep=comma] {data/real-world_reward.csv};
}
\end{axis}
\end{tikzpicture}

%% file: figures/plot_performance.tex
\begin{tikzpicture}
\pgfplotsset{
  width=0.55\linewidth,
  height=0.4\linewidth,
  xtick pos=left,
  ytick pos=left,
  tick label style={font=\tiny},
  ymajorgrids=true,
}

\begin{axis}[ 
  at={(0, 0)},
  title style={yshift=-.7ex},
  title={\textsc{Synthetic Data}},
  ylabel absolute, ylabel style={ yshift=-.4cm},
  ylabel={\footnotesize{Performance}},
  xlabel near ticks,
  legend style={
    align=left,
    at={(2.5, 1.6)},
    legend columns=5,
    font=\scriptsize, 
    draw=none,
    fill=none,
    /tikz/every even column/.append style={column sep=0.15cm}},
  ymin=-70,
  ymax=70,
  xmin=0,
  xmax=500,
]


\addplot [green, thick, mark=none] table [x=t, y=decomposed_avg, col sep=comma] {data/smooth_performance_synthetic_n25_b1.csv}; \addlegendentry{LIZARD}

\addplot [red, thick, mark=none] table [x=t, y=CUCB_avg, col sep=comma] {data/smooth_performance_synthetic_n25_b1.csv}; \addlegendentry{CUCB}

\addplot [orange, thick, mark=none] table [x=t, y=lipschitz_avg, col sep=comma] {data/smooth_performance_synthetic_n25_b1.csv}; \addlegendentry{Zooming}

\addplot [yellow, thick, mark=none] table [x=t, y=minion_avg, col sep=comma] {data/smooth_performance_synthetic_n25_b1.csv}; \addlegendentry{MINION}

\addplot [cyan, dashed, thick, mark=none] table [x=t, y=historical_exploit_avg, col sep=comma] {data/smooth_performance_synthetic_n25_b1.csv}; \addlegendentry{Exploit}


\addplot [draw=none, mark=none, name path=exploit_high] table [x=t, y expr=\thisrowno{3}+\thisrowno{4}, col sep=comma] {data/smooth_performance_synthetic_n25_b1.csv};
\addplot [draw=none, mark=none, name path=exploit_low] table [x=t, y expr=\thisrowno{3}-\thisrowno{4}, col sep=comma] {data/smooth_performance_synthetic_n25_b1.csv};
\tikzfillbetween[of=exploit_low and exploit_high, on layer=bg]{cyan, opacity=0.2};

\addplot [draw=none, mark=none, name path=minion_high] table [x=t, y expr=\thisrowno{11}+\thisrowno{12}, col sep=comma] {data/smooth_performance_synthetic_n25_b1.csv};
\addplot [draw=none, mark=none, name path=minion_low] table [x=t, y expr=\thisrowno{11}-\thisrowno{12}, col sep=comma] {data/smooth_performance_synthetic_n25_b1.csv};
\tikzfillbetween[of=minion_low and minion_high, on layer=main]{yellow, opacity=0.5};

\addplot [draw=none, mark=none, name path=zooming_high] table [x=t, y expr=\thisrowno{13}+\thisrowno{14}, col sep=comma] {data/smooth_performance_synthetic_n25_b1.csv};
\addplot [draw=none, mark=none, name path=zooming_low] table [x=t, y expr=\thisrowno{13}-\thisrowno{14}, col sep=comma] {data/smooth_performance_synthetic_n25_b1.csv};
\tikzfillbetween[of=zooming_low and zooming_high, on layer=main]{orange, opacity=0.3};

\addplot [draw=none, mark=none, name path=cucb_high] table [x=t, y expr=\thisrowno{9}+\thisrowno{10}, col sep=comma] {data/smooth_performance_synthetic_n25_b1.csv};
\addplot [draw=none, mark=none, name path=cucb_low] table [x=t, y expr=\thisrowno{9}-\thisrowno{10}, col sep=comma] {data/smooth_performance_synthetic_n25_b1.csv};
\tikzfillbetween[of=cucb_low and cucb_high, on layer=main]{red, opacity=0.2};

\addplot [draw=none, mark=none, name path=lizard_high] table [x=t, y expr=\thisrowno{5}+\thisrowno{6}, col sep=comma] {data/smooth_performance_synthetic_n25_b1.csv};
\addplot [draw=none, mark=none, name path=lizard_low] table [x=t, y expr=\thisrowno{5}-\thisrowno{6}, col sep=comma] {data/smooth_performance_synthetic_n25_b1.csv};
\tikzfillbetween[of=lizard_low and lizard_high, on layer=bg]{green, opacity=0.2};

\end{axis}

\begin{axis}[ 
  at={(0.47\linewidth, 0)},
  title={\textsc{Real-World Data}}, 
  title style={yshift=-.7ex},
  xlabel absolute, xlabel style={xshift=-2cm, yshift=.9ex},
  xlabel={\footnotesize{Timestep~$t$}},
  ymin=-100,
  ymax=100,
  xmin=0,
  xmax=500,
]


\addplot [draw=none, mark=none, name path=exploit_high] table [x=t, y expr=\thisrowno{3}+\thisrowno{4}, col sep=comma] {data/smooth_performance_real_n25_b1.csv};
\addplot [draw=none, mark=none, name path=exploit_low] table [x=t, y expr=\thisrowno{3}-\thisrowno{4}, col sep=comma] {data/smooth_performance_real_n25_b1.csv};
\tikzfillbetween[of=exploit_low and exploit_high, on layer=bg]{cyan, opacity=0.2};

\addplot [cyan, dashed, thick, mark=none] table [x=t, y=historical_exploit_avg, col sep=comma] {data/smooth_performance_real_n25_b1.csv}; 

\addplot [draw=none, mark=none, name path=minion_high] table [x=t, y expr=\thisrowno{11}+\thisrowno{12}, col sep=comma] {data/smooth_performance_real_n25_b1.csv};
\addplot [draw=none, mark=none, name path=minion_low] table [x=t, y expr=\thisrowno{11}-\thisrowno{12}, col sep=comma] {data/smooth_performance_real_n25_b1.csv};
\tikzfillbetween[of=minion_low and minion_high, on layer=main]{yellow, opacity=0.5};

\addplot [yellow, thick, mark=none] table [x=t, y=minion_avg, col sep=comma] {data/smooth_performance_real_n25_b1.csv}; 

\addplot [draw=none, mark=none, name path=zooming_high] table [x=t, y expr=\thisrowno{13}+\thisrowno{14}, col sep=comma] {data/smooth_performance_real_n25_b1.csv};
\addplot [draw=none, mark=none, name path=zooming_low] table [x=t, y expr=\thisrowno{13}-\thisrowno{14}, col sep=comma] {data/smooth_performance_real_n25_b1.csv};
\tikzfillbetween[of=zooming_low and zooming_high, on layer=main]{orange, opacity=0.2=3};

\addplot [orange, thick, mark=none] table [x=t, y=lipschitz_avg, col sep=comma] {data/smooth_performance_real_n25_b1.csv}; 

\addplot [draw=none, mark=none, name path=cucb_high] table [x=t, y expr=\thisrowno{9}+\thisrowno{10}, col sep=comma] {data/smooth_performance_real_n25_b1.csv};
\addplot [draw=none, mark=none, name path=cucb_low] table [x=t, y expr=\thisrowno{9}-\thisrowno{10}, col sep=comma] {data/smooth_performance_real_n25_b1.csv};
\tikzfillbetween[of=cucb_low and cucb_high, on layer=main]{red, opacity=0.3};

\addplot [red, thick, mark=none] table [x=t, y=CUCB_avg, col sep=comma] {data/smooth_performance_real_n25_b1.csv}; 

\addplot [draw=none, mark=none, name path=lizard_high] table [x=t, y expr=\thisrowno{5}+\thisrowno{6}, col sep=comma] {data/smooth_performance_real_n25_b1.csv};
\addplot [draw=none, mark=none, name path=lizard_low] table [x=t, y expr=\thisrowno{5}-\thisrowno{6}, col sep=comma] {data/smooth_performance_real_n25_b1.csv};
\tikzfillbetween[of=lizard_low and lizard_high, on layer=bg]{green, opacity=0.2};

\addplot [green, thick, mark=none] table [x=t, y=decomposed_avg, col sep=comma] {data/smooth_performance_real_n25_b1.csv}; 

\end{axis}

\end{tikzpicture}

%% file: figures/plot_vary_params.tex
\begin{tikzpicture}
\pgfplotsset{
  width=.4\linewidth,
  height=.36\linewidth,
  xtick pos=left,
  ytick pos=left,
  tick label style={font=\tiny},
  ymajorgrids=true,
}

\begin{axis}[ 
  at={(0, 0\linewidth)},
  ylabel={\footnotesize{Reward}},
  xlabel near ticks,
  ylabel near ticks,
  legend style={
    align=left,
    at={(1.2,1.55)},
    legend columns=2,
    row sep=0cm,
    font=\scriptsize, 
    draw=none,
    fill=none,},
  xmin=0,
  xmax=500,
]
\addplot [red, thin, mark=none] table [x=t, y=L20, col sep=comma] {data/reward_vary_L.csv}; \addlegendentry{$2.0$}
\addplot [brown, thin, mark=none] table [x=t, y=L12, col sep=comma] {data/reward_vary_L.csv}; \addlegendentry{$1.2$}
\addplot [green, thin, mark=none] table [x=t, y=L10, col sep=comma] {data/reward_vary_L.csv}; \addlegendentry{$1$}
\addplot [blue, thin, mark=none] table [x=t, y=L_exact, col sep=comma] {data/reward_vary_L.csv}; \addlegendentry{exact}

\addplot [draw=none, mark=none, name path=L20_high] table [x=t, y expr=\thisrowno{13}+\thisrowno{14}, col sep=comma] {data/reward_vary_L.csv};
\addplot [draw=none, mark=none, name path=L20_low] table [x=t, y expr=\thisrowno{13}-\thisrowno{14}, col sep=comma] {data/reward_vary_L.csv};
\tikzfillbetween[of=L20_low and L20_high, on layer=bg]{red, opacity=0.2};

\addplot [draw=none, mark=none, name path=L12_high] table [x=t, y expr=\thisrowno{9}+\thisrowno{10}, col sep=comma] {data/reward_vary_L.csv};
\addplot [draw=none, mark=none, name path=L12_low] table [x=t, y expr=\thisrowno{9}-\thisrowno{10}, col sep=comma] {data/reward_vary_L.csv};
\tikzfillbetween[of=L12_low and L12_high, on layer=bg]{brown, opacity=0.2};

\addplot [draw=none, mark=none, name path=L10_high] table [x=t, y expr=\thisrowno{7}+\thisrowno{8}, col sep=comma] {data/reward_vary_L.csv};
\addplot [draw=none, mark=none, name path=L10_low] table [x=t, y expr=\thisrowno{7}-\thisrowno{8}, col sep=comma] {data/reward_vary_L.csv};
\tikzfillbetween[of=L10_low and L10_high, on layer=bg]{green, opacity=0.2};

\addplot [draw=none, mark=none, name path=exact_high] table [x=t, y expr=\thisrowno{15}+\thisrowno{16}, col sep=comma] {data/reward_vary_L.csv};
\addplot [draw=none, mark=none, name path=exact_low] table [x=t, y expr=\thisrowno{15}-\thisrowno{16}, col sep=comma] {data/reward_vary_L.csv};
\tikzfillbetween[of=exact_low and exact_high, on layer=bg]{blue, opacity=0.2};
\end{axis}

\begin{axis}[ 
  at={(.31\linewidth, 0)},
  xlabel absolute, xlabel style={xshift=0, yshift=-1},
  xlabel={\footnotesize{Timestep~$t$}},
  legend style={
    align=left,
    at={(1.2,1.55)},
    legend columns=1,
    font=\scriptsize, 
    draw=none,
    fill=none,},
  xmin=0,
  xmax=500,
]
\addplot [green, thin, mark=none] table [x=t, y=mono, col sep=comma] {data/reward_vary_mono.csv}; \addlegendentry{monotonic}
\addplot [magenta, thin, mark=none] table [x=t, y=no_mono, col sep=comma] {data/reward_vary_mono.csv}; \addlegendentry{no monotonic}


\addplot [draw=none, mark=none, name path=mono_high] table [x=t, y expr=\thisrowno{3}+\thisrowno{4}, col sep=comma] {data/reward_vary_mono.csv};
\addplot [draw=none, mark=none, name path=mono_low] table [x=t, y expr=\thisrowno{3}-\thisrowno{4}, col sep=comma] {data/reward_vary_mono.csv};
\tikzfillbetween[of=mono_low and mono_high, on layer=bg]{green, opacity=0.2};

\addplot [draw=none, mark=none, name path=no_mono_high] table [x=t, y expr=\thisrowno{1}+\thisrowno{2}, col sep=comma] {data/reward_vary_mono.csv};
\addplot [draw=none, mark=none, name path=no_mono_low] table [x=t, y expr=\thisrowno{1}-\thisrowno{2}, col sep=comma] {data/reward_vary_mono.csv};
\tikzfillbetween[of=no_mono_low and no_mono_high, on layer=bg]{magenta, opacity=0.2};

\end{axis}

\begin{axis}[ 
  at={(.62\linewidth, 0)},
  xlabel near ticks,
  ytick={.5, .6},
  legend style={
    align=left,
    at={(1.2,1.55)},
    legend columns=2,
    font=\scriptsize, 
    draw=none,
    fill=none,},
  xmin=0,
  xmax=45,
]

\addplot [gray, thin, mark=none] table [x=t, y=H0, col sep=comma] {data/reward_vary_H.csv}; \addlegendentry{0}
\addplot [green, thin, mark=none] table [x=t, y=H50, col sep=comma] {data/reward_vary_H.csv}; \addlegendentry{50}
\addplot [purple, thin, mark=none] table [x=t, y=H100, col sep=comma] {data/reward_vary_H.csv}; \addlegendentry{100}
\addplot [blue, thin, mark=none] table [x=t, y=H500, col sep=comma] {data/reward_vary_H.csv}; \addlegendentry{500}

\addplot [draw=none, mark=none, name path=H0_high] table [x=t, y expr=\thisrowno{1}+\thisrowno{2}, col sep=comma] {data/reward_vary_H.csv};
\addplot [draw=none, mark=none, name path=H0_low] table [x=t, y expr=\thisrowno{1}-\thisrowno{2}, col sep=comma] {data/reward_vary_H.csv};
\tikzfillbetween[of=H0_low and H0_high, on layer=main]{gray, opacity=0.2};

\addplot [draw=none, mark=none, name path=H100_high] table [x=t, y expr=\thisrowno{5}+\thisrowno{6}, col sep=comma] {data/reward_vary_H.csv};
\addplot [draw=none, mark=none, name path=H100_low] table [x=t, y expr=\thisrowno{5}-\thisrowno{6}, col sep=comma] {data/reward_vary_H.csv};
\tikzfillbetween[of=H100_low and H100_high, on layer=main]{purple, opacity=0.1};

\addplot [draw=none, mark=none, name path=H500_high] table [x=t, y expr=\thisrowno{7}+\thisrowno{8}, col sep=comma] {data/reward_vary_H.csv};
\addplot [draw=none, mark=none, name path=H500_low] table [x=t, y expr=\thisrowno{7}-\thisrowno{8}, col sep=comma] {data/reward_vary_H.csv};
\tikzfillbetween[of=H500_low and H500_high, on layer=main]{blue, opacity=0.1};

\addplot [draw=none, mark=none, name path=H50_high] table [x=t, y expr=\thisrowno{3}+\thisrowno{4}, col sep=comma] {data/reward_vary_H.csv};
\addplot [draw=none, mark=none, name path=H50_low] table [x=t, y expr=\thisrowno{3}-\thisrowno{4}, col sep=comma] {data/reward_vary_H.csv};
\tikzfillbetween[of=H50_low and H50_high, on layer=main]{green, opacity=0.2};

\end{axis}

\node[anchor=north] (a1) at (.11\linewidth, -.25){\scriptsize{(a)}};
\node[anchor=north] (a1) at (.42\linewidth, -.25){\scriptsize{(b)}};
\node[anchor=north] (a1) at (.73\linewidth, -.25){\scriptsize{(c)}};

\end{tikzpicture}

%% file: appendix.tex
\appendix

\section{Proof of Theorem 2}

\fixedDiscretization*

\begin{proof}
We first analyze the discretization regret.
We then bound the arm selection regret by treating the problem as a combinatorial multi-armed bandit problem and apply a theorem given by \citet{chen2016combinatorial}.
Combining these two regret terms gives us the overall regret bound for the fixed discretization algorithm.

\paragraph{Discretization Regret}
Using discretization levels $\discretization$ with a minimum gap $\discretizationgap$, let $\text{OPT}_{\discretization}$ be the true optimum value and $\opteffortvec_{\discretization}$ be the optimal solution when the rewards of all discretized points are fully known.
Let the true optimum and the optimal solution without discretization be $\text{OPT}$ and $\opteffortvec$.
Then we have
\begin{small}
\begin{align}
    \text{OPT} &= \mu(\opteffortvec) = \sum_{i=1}^N \mu_i(\opteffort{i}) \nonumber \\
    &\leq \sum_{i=1}^N \left( \mu_i( \left \lfloor{\opteffort{i}} \right \rfloor_{\discretization}) + L \lvert \opteffort{i} - \left \lfloor{\opteffort{i}} \right \rfloor_{\discretization} \rvert \right) \nonumber \\
    & \leq \sum_{i=1}^N \left( \mu_i( \left \lfloor{\opteffort{i}} \right \rfloor_{\discretization}) + L \discretizationgap \right) \nonumber \\
    & = \sum_{i=1}^N \mu_i( \left \lfloor{\opteffort{i}} \right \rfloor_{\discretization}) + \sum_{i=1}^N L \discretizationgap \nonumber \\ 
    & \leq \text{OPT}_\discretization + N L \discretizationgap \ .
\label{eqn:discrete_continuous_gap}
\end{align}
\end{small}%
This yields $\text{OPT} - \text{OPT}_{\discretization} \leq N L\discretizationgap$, which is the discretization regret incurred by the given discretization.

\paragraph{Selection Regret}
Given a discretization $\discretization$, we treat each target with a specified patrol effort as a base arm.
If we assume the smallest discretization gap is $\discretizationgap$, we can bound the total number of base arms by $N / \discretizationgap$.
We denote a feasible patrol coverage as a super arm: a set of base arms that must satisfy the budget constraints and the restriction of selecting exactly one effort level per target.
The set of all feasible super arms implicitly forms a feasibility constraint, which fits into the context of the combinatorial multi-armed bandit problem \cite{chen2016combinatorial}.
In addition, our Lipschitz reward function also satisfies the bounded smoothness condition with
\begin{small}
\begin{align}
& \lvert \mu(\effortvec) - \mu'(\effortvec) \rvert = \left\lvert \sum_{i} \mu_i(\effort{i}) - \mu'_i(\effort{i}) \right\rvert \nonumber \\
& \leq N \max_{i, \effort{i}} \lvert \mu_i(\effort{i}) - \mu'_i(\effort{i}) \rvert = N \Vert \mu - \mu' \Vert_{\infty} = f(\Lambda) \label{eqn:monotone_function}
\end{align}
\end{small}%
where $f(\Lambda) = N \Lambda, \Lambda = \Vert \mu - \mu' \Vert_{\infty}$ is a linear function of $\Lambda$.
So we can apply the regret bound of combinatorial multi-armed bandit with polynomial bounded smoothness function $f(x) = \gamma x^\omega$ (Theorem 2 given by \citeauthor{chen2016combinatorial})
\begin{small}
\begin{align}
    \text{Reg}(T)
    \leq \frac{2\gamma}{2 - \omega} (6 m \log{T})^{\frac{\omega}{2}} \cdot T^{1 - \frac{\omega}{2}} + \left( \frac{\pi^2}{3} + 1 \right) m \maxregret \nonumber
\end{align}
\end{small}%
where $\omega = 1, \gamma = N$, $m$ is the number of base arms, and $\maxregret$ is the maximum regret that can be achieved.
In our domain, the number of base arms is bounded by $m \leq N / \discretizationgap$.
The maximum regret is bounded by the maximum reward achievable, which can be bounded by monotonicity and Lipschitz-continuity:
\begin{small}
\begin{align*}
\maxregret \leq \mu(\opteffortvec) = \sum_{i} \mu_i(\opteffort{i}) \leq N L \ .
\end{align*}
\end{small}%
By substituting $\omega=1$, $\gamma=N$, and using upper bounds for $m$ and $\maxregret$, we get a valid regret bound of Algorithm~\ref{alg:decomposed} with fixed discretization gap $\discretizationgap$:
\begin{small}
\begin{align}\label{eqn:discrete_regret_bound}
\text{Reg}_{\discretizationgap}^{\text{discrete}}(T) \leq 2 N \sqrt{\frac{6NT \log{T}}{\discretizationgap}} + \left( \frac{\pi^2}{3} + 1 \right) \frac{N^2L}{\discretizationgap} \ .
\end{align}
\end{small}%
This regret is defined with respect to the optimum over all feasible discrete selections.
Therefore, we can combine Inequalities~\ref{eqn:discrete_continuous_gap} and~\ref{eqn:discrete_regret_bound} to derive the regret with respect to the true optimum
\begin{footnotesize}
\begin{align}
\text{Reg}_{\discretizationgap}(T) & \leq 2 N \sqrt{\frac{6NT \log{T}}{\discretizationgap}} + \left( \frac{\pi^2}{3} + 1 \right) \frac{N^2L}{\discretizationgap} + NL\discretizationgap T \nonumber \\
& = O \left( \sqrt{\frac{N^3 T \log{T}}{\discretizationgap}} + \frac{N^2L}{\discretizationgap} + NL\discretizationgap T \right) \label{eqn:fixed_regret_bound_clear}
\end{align}
\end{footnotesize}%
where the gap in Inequality~\ref{eqn:discrete_continuous_gap} is for each iteration, so must be multiplied by the number of timesteps $T$.
\end{proof}

\section{Proof of Theorem~\ref{thm:adaptive_discretization}}

\adaptiveDiscretization*

\begin{proof}
We want to set $2 N \sqrt{6NT \log{T}/\discretizationgap}$ and $NL \discretizationgap T$ in Eq.~\ref{eqn:fixed_regret_bound_clear} to be of the same order.
This is equivalent to solving
\begin{small}
\begin{align*}
    & O \left( N \sqrt{N T \log{T} / \discretizationgap} \right) = O(NL \discretizationgap T) \\
    \iff & O \left( \frac{T}{\log{T}} \right) = O \left( \frac{N}{L^2 \discretizationgap^3} \right) \ .
\end{align*}
\end{small}%
Applying the fact $O \left( T / \log{T} \right) = c \quad \Longrightarrow \quad T = O(c \log{c})$ for any constant $c$ yields
\begin{small}
\begin{align}
T_\discretizationgap = O \left( \frac{N}{L^2 \discretizationgap^3} \log{\frac{N}{L^2 \discretizationgap^3}} \right) \nonumber
\end{align}
\end{small}%
which is the optimal stopping condition to let the two terms of the regret bound in Eq.~\ref{eqn:fixed_regret_bound_clear} be of the same order.
We set
\begin{small}
\begin{align}
T_\discretizationgap = \frac{N}{L^2 \discretizationgap^3} \log{\frac{N}{L^2 \discretizationgap^3}} \ . \nonumber
\end{align}
\end{small}%
Substitute this into Eq.~\ref{eqn:fixed_regret_bound_clear} with $t \leq T_\discretizationgap = \frac{N}{L^2 \discretizationgap^3} \log{\frac{N}{L^2 \discretizationgap^3}}$ and $\log{t} \leq \log{T_\discretizationgap} = \log{\left( \frac{N}{L^2 \discretizationgap^3} \log{\frac{N}{L^2 \discretizationgap^3}} \right)} \leq 2 \log \frac{N}{L^2 \discretizationgap^3}$ to get
\begin{small}
\begin{align}
& \quad \text{Reg}_{\discretizationgap}(t) \leq \text{Reg}_{\discretizationgap}(T_\discretizationgap) \nonumber \\
& \leq 2 N \sqrt{\frac{6NT_\discretizationgap \log{T_\discretizationgap}}{\discretizationgap}} + \left( \frac{\pi^2}{3} + 1 \right) \frac{N^2L}{\discretizationgap} + NL\discretizationgap T_\discretizationgap \nonumber \\
& = (4 \sqrt{3} + 1) \frac{N^2 \log{\frac{N}{L^2 \discretizationgap^3}}}{L \discretizationgap^2} + \left( \frac{\pi^2}{3} + 1 \right) \frac{N^2L}{\discretizationgap} \ . \nonumber 
\end{align}
\end{small}%
Due to the discretization regret term, the bound given by Eq.~\ref{eqn:fixed_regret_bound_clear} is not a sublinear function in $T$, so is not no-regret.
To achieve no-regret, we need to adaptively adjust the discretization levels.
Therefore, as described in Algorithm~\ref{alg:adaptive_decomposed}, we gradually reduce the discretization gap $\discretizationgap_i = 2^{-i}$ depending on the number of timesteps elapsed.

For total timesteps $T$, let $k \in \N \cup \{0\}$ such that $\sum_{i=0}^{k-1} T_{\discretizationgap_i} \leq T \leq \sum_{i=0}^{k} T_{\discretizationgap_i}$.
Then we can bound the regret at time $T$ by the regret at time $\sum_{i=0}^{k} T_{\discretizationgap_i}$, which yields
\begin{footnotesize}
\begin{align}
    & \quad \text{Reg}(T) \leq \text{Reg} \left( \sum_{i=0}^{k} T_{\discretizationgap_i} \right) \leq \sum_{i=0}^k \text{Reg}_{\discretizationgap_i}(T_{\discretizationgap_i}) \nonumber \\
    & \leq \sum\limits_{i=0}^k \left( (4 \sqrt{3} + 1) \frac{N^2 \log{\frac{N}{L^2 \discretizationgap_i^3}}}{L \discretizationgap_i^2} + \left( \frac{\pi^2}{3} + 1 \right) \frac{N^2L}{\discretizationgap_i} \right) \nonumber \\
    & \leq \sum\limits_{i=0}^k \left( (4 \sqrt{3} + 1) \frac{N^2}{L} 2^{2i} \log{ \left( \frac{N}{L^2} 2^{3i} \right) } + \left( \frac{\pi^2}{3} + 1 \right) N^2L 2^i \right) \nonumber \\
    & \leq (4 \sqrt{3} + 1) \frac{N^2}{L} 2^{2k+2} \log{\left( \frac{N}{L^2} 2^{3k+3} \right)} + \left( \frac{\pi^2}{3} + 1 \right) N^2L 2^{k+1} \ . \nonumber
\end{align}
\end{footnotesize}%
Since we know
\begin{footnotesize}
\begin{align}
    T & = O \left( \sum_{i=0}^{k} T_{\discretizationgap_i} \right) \nonumber \\
    \Longrightarrow T & = O \left( \sum_{i=0}^{k} \frac{N}{L^2 \discretizationgap_i^3} \log{\frac{N}{L^2 \discretizationgap_i^3}} \right) \nonumber \\
    \Longrightarrow T & = O\left(\sum_{i=0}^{k} \frac{N}{L^2} 2^{3i} \log{\left( \frac{N}{L^2} 2^{3i} \right)} \right) \nonumber \\
    \Longrightarrow T & = O \left( \frac{N}{L^2} 2^{3k+3} \log{\left( \frac{N}{L^2} 2^{3k+3} \right)} \right) \ , \nonumber
\end{align}
\end{footnotesize}%
we therefore have $T = O(\frac{N}{L^2} 2^{3k} \log{(\frac{N}{L^2} 2^{3k})})$, which yields $2^{3k} = O(\frac{L^2 T}{N \log{T}})$.
Replacing all $2^{k}$ with $O((\frac{L^2 T}{N \log{T}})^{\frac{1}{3}})$ provides a regret bound dependent on $T$ and $N$ only:
\begin{small}
\begin{align}
    & \text{Reg}(T) \leq O \left( \frac{N^2}{L} 2^{2k} \log{\left( \frac{N}{L^2} 2^{3k} \right)} + N^2 L 2^k \right) \nonumber \\
    & \leq O \left(\frac{N^2}{L} \left( \frac{L^2 T}{N \log{T}} \right)^{\frac{2}{3}} \log{\left( \frac{T}{\log{T}} \right)} + N^2 L \left(\frac{L^2 T}{N \log{T}} \right)^{\frac{1}{3}} \right) \nonumber \\
    & \leq O \left(L^{\frac{1}{3}} N T^{\frac{2}{3}} (\log{T})^{\frac{1}{3}} + L^{\frac{5}{3}} N^{\frac{5}{3}}  T^{\frac{1}{3}} (\log{T})^{-\frac{1}{3}} \right) \ .
    \label{eqn:adaptive_regret_bound_clear}
\end{align}
\end{small}%
We only care about the regret order in terms of the dominating variable $T$, so the second term is dominated by the first term.
Therefore, $\text{Reg}(T) \leq O\left(L^{\frac{1}{3}} N T^{\frac{2}{3}} (\log{T})^{\frac{1}{3}} \right)$.
\end{proof}

\lipschitzCUCB*
To prove Theorem~\ref{thm:lipschitz_cucb}, we first describe a weaker version in Lemma~\ref{thm:confidence_bounds}.
Lemma~\ref{thm:confidence_bounds} shows that if all arms are sufficiently sampled, the probability that we hit a bad super arm is small.
Lemma~\ref{thm:confidence_bounds} attributes all the bad super arm selections with error greater than $\minregret$ to a maximum regret $\maxregret$, which overestimates the overall regret and can be tightened with a more careful analysis (omitted here due to space).

\begin{restatable}[]{lemma}{confidenceBounds}\label{thm:confidence_bounds}
Given the prior knowledge of Lipschitz-continuity, monotonicity, and the definition of $\ucb$ in Sec.~\ref{sec:algorithm}, the regret bound of CUCB algorithm~\cite{chen2016combinatorial} holds with
\begin{align}\label{eqn:tighter_confidence_bound}
\text{Reg}(T) \leq & \left( \frac{6 m \log{T}}{(f^{-1}(\minregret))^2} + \frac{\pi^2}{3} m^2 + m \right) \maxregret \ .
\end{align}
\end{restatable}
\begin{proof}
The main proof relies on whether the tighter UCB defined in Sec.~\ref{sec:algorithm} works for the combinatorial multi-armed bandit problem.
Specifically, we must confirm that the proof of CUCB given by \citet{chen2016combinatorial} works with a different confidence bound definition.

Define $T_{i,t}$ as the number of pulls of arm $i$ up to time $t$.
We assume there are $m$ total base arms.
Let $\mathbf{\mu} = [\mu_1, \mu_2, \ldots, \mu_m]$ be the vector of true expectations of all base arms.
Let $X_{i,t}$ be the revelation of arm $i$ from a pull at time~$t$ and $\bar{X}_{i,s} = (\sum_{j=1}^s X_{i,j})/s$ be the average reward of the first $s$ pulls of arm $i$.
Therefore, we write $\bar{X}_{i,T_{i,t}}$ to represent the average reward of arm $i$ at time $t$ as defined in Sec.~\ref{sec:algorithm}.
Let $S_t$ be the super arm pulled in round~$t$.
We say round~$t$ is bad if $S_t$ is not the optimal arm.
Define event $E_t = \{ \forall i \in [m], \lvert \bar{X}_{i, T_{i,t-1}} - \mu_i \rvert \leq r'_t(i) \}$, where recall that $r_t(i) = \sqrt{3\log{t} (2 T_{i,t-1})^{-1}}$ is the standard confidence radius, and $r'_t(i) = \ucb_t(i) - \bar{X}_{i, T_{i,t-1}} \leq \selfucb_t(i) - \bar{X}_{i, T_{i,t-1}} = r_t(i)$ is the tighter confidence bound we introduce in Sec.~\ref{sec:algorithm}.
We want to bound the probability that all the sampled averages are close to the true mean value $\mu_i$ with distance at most the confidence radius $r'_t(i) ~\forall i$.

At time $t$, the probability that the sampled mean $\bar{X}_{i,T_{i,t-1}}$ of arm $i$ lies within the ball with center $\mu_i$ and radius $r_t(i)$ is bounded by $1 - 2 t^{-3}$ (double-sided Hoeffding bound), which characterizes the probability of having one single $\selfucb$ bound hold for arm $i$.
Moreover, at a fixed time $t$, if we have all $m$ $\selfucb$ bounds hold, then all the $m$ $\ucb$ bounds will automatically hold because each bound introduced by the Lipschitz continuity from any arm is a valid bound and thus $\ucb$ bound, the minimum of all the valid bounds, must also hold.
Therefore, using the notation of $\selfucb$ radius $r'_t(i)$, we have

\begin{small}
\begin{align}
\text{prob} \{ \neg E_t \} = & \text{prob}\{ \lvert \bar{X}_{i, T_{i,t-1}} - \mu_i \rvert \leq r'_t(i) ~\forall i \in [m] \} \nonumber \\
\geq & \text{prob}\{ \lvert \bar{X}_{i, s} - \mu_i \rvert \leq r'_t(i) ~\forall i \in [m], ~\forall s \in [t] \} \nonumber \\
\geq & \text{prob}\{ \lvert \bar{X}_{i, s} - \mu_i \rvert \leq r_t(i) ~\forall i \in [m], ~\forall s \in [t] \} \nonumber \\
\geq & 1 - m t \cdot 2 t^{-3} \nonumber \\
= & 1 - 2 m t^{-2} \nonumber
\end{align}
\end{small}%
where the second inequality is due to the fact that $\lvert \bar{X}_{i, s} - \mu_i \rvert \leq r_t(i) ~\forall i \in [m] \Rightarrow \lvert \bar{X}_{i, s} - \mu_i \rvert \leq r'_t(i) ~\forall i \in [m]$ under the Lipschitz condition, and the last inequality is due to the union bound across all $m$ arms and $t$ time steps. Equivalently, we have $ \neg \text{prob} \{ \neg E_t \} \leq 2 m t^{-2}$.

According to the definition $r'_t(i) = \ucb_t(i) - \bar{X}_{i, T_{i,t-1}}$, $E_t$ is true (thus $|\bar{X}_{i,T_{i,t-1}} - \mu_i| \leq r'_t(i)$) implies $| \ucb_t(i) - \mu_i | \leq 2 r'_t(i) $ and $\ucb_t(i) \geq \mu_i ~\forall i$.
Let $\Lambda \coloneqq \sqrt{\frac{3 \log{t}}{2l_t}}$ and $\Lambda'_t \coloneqq \max_{i \in S_t} r'_t(i) \leq \max_{i \in S_t} r_t(i) \coloneqq \Lambda_t$.
Then we have
\begin{small}
\begin{align}
    & E_t \Rightarrow ~\forall i \in S_t, \lvert \ucb_t(i) - \mu_i \rvert \leq 2 r'_t(i) \leq 2 \Lambda'_t \nonumber \\
    & \{ S_t \text{ is bad}, ~\forall i \in S_t, T_{i,t-1} > l_t \} \Rightarrow \Lambda > \Lambda_t \geq \Lambda'_t \ . \nonumber
\end{align}
\end{small}%
Let $\ucb_t = [\ucb_t(1), \ucb_t(2), \dots, \ucb_t(m)]$ be a vector of UCBs of all subarms.
If the above two events on the left-hand side both hold $\{E_t, S_t \text{ is bad}, ~\forall i \in S_t, T_{i,t-1} > l_t \}$ at time $t$, if we denote $\reward_\mu(S) = \sum_{i \in S} \mu_i$, we have
\begin{small}
\begin{align}
& \reward_{\mathbf{\mu}}(S_t) + f(2 \Lambda)
> \reward_\mathbf{\mu}(S_t) + f(2 \Lambda'_t ) \nonumber \\
\geq & \reward_{\ucb_t}(S_t) = \text{opt}_{\ucb_t} \geq \reward_{\ucb_t}(S^\star_{\mathbf{\mu}}) \nonumber \\
\geq & \reward_{\mathbf{\mu}}(S^\star_{\mathbf{\mu}}) \geq \text{opt}_{\mathbf{\mu}} \label{eqn:appendix_selection_bound}
\end{align}
\end{small}%
where the first inequality is due to the monotonicity of function~$f$ and $\Lambda > \Lambda'_t$; 
the second is due to the definition of bounded smoothness function $f$ in Eq.~\ref{eqn:monotone_function} and the condition that $ E_t \Rightarrow \lvert \ucb_t(i) - \mu_i \rvert \leq 2 \Lambda'_t ~\forall i \in S_t$; 
the third is due to the optimal selection of $S_t$ with respect to $\ucb_t$; 
the fourth is due to the optimality of $\text{opt}_{\ucb_t}$, where $S^\star_{\mathbf{\mu}}$ is the true optimal solution with respect to the real $\mathbf{\mu}$; 
the fifth is due to the upper confidence bound between $\ucb_t \geq \mathbf{\mu}$ when $E_t$ is true; and the sixth is again due to the optimality of $S^\star_{\mathbf{\mu}}$.

By substituting $l_t = \frac{6 \log{t}}{ (f^{-1}(\minregret))^2}$ into $\Lambda = \sqrt{\frac{3\log{t}}{2 l_t}}$, we have $f(2 \Lambda) = \minregret$.
Eq.~\ref{eqn:appendix_selection_bound} contradicts the definition of minimum regret $\minregret$, the smallest non-zero regret (since $S_t$ is bad thus not optimal but it is also closer to the optimal with distance smaller than $\minregret$).
That is:
\begin{small}
\begin{align}
    & \text{prob}\{ E_t, S_t \text{ is bad}, ~\forall i \in S_t, T_{i,t-1} > l_t \} = 0 \nonumber \\
    \Rightarrow & \text{prob}\{ S_t \text{ is bad}, ~\forall i \in S_t, T_{i,t-1} > l_t \} \leq \text{prob}\{\neg E_t\} \leq 2 m t^{-2} \ .\nonumber
\end{align}
\end{small}%
Lastly, we bound the total number of bad arm pulls by the probability that $E_t$ did not happen and a time-dependent term as shown by \citet{chen2016combinatorial} (proof omitted for space):
\begin{small}
\begin{align}
    \mathbb{E} \left[ \text{\# bad arm pulls} \right] \leq & m (l_T+1)  + \sum_{i=1}^m 2m t^{-2} \nonumber \\
    \leq & \frac{6 m \log{T}}{(f^{-1}(\minregret))^2} + m + \frac{\pi^2}{3} m \ . \nonumber
\end{align}
\end{small}%
We then bound the cumulative regret by attributing the maximum regret to each of the bad arm pulls, which yields:
\begin{small}
\begin{align*}
    \text{Reg}(T) \leq & \left( \frac{6 m \log{T}}{(f^{-1}(\minregret))^2} + m + \frac{\pi^2}{3} m \right) \maxregret \ . \qedhere
\end{align*}
\end{small}%
\end{proof}

\begin{proof}[Proof of Theorem~\ref{thm:lipschitz_cucb}]
As shown by \citet{chen2016combinatorial}, the first term in Eq.~\ref{eqn:tighter_confidence_bound} can be further broken down by having a finer regret attribution.
Now we assign a maximum regret to all the bad super arm pulls, while we can in fact more carefully analyze the cumulative regret by counting the number of bad arm pulls resulting in a certain amount of regret.
The only part different between ours and \citet{chen2016combinatorial} is the last term in Eq.~\ref{eqn:tighter_confidence_bound}, which is not affected by the additional analysis.
So we conclude that the same argument still applies, where when the bounded smoothness function is of the form $f(x) = \gamma x^\omega$, we can get a similar result:
\begin{small}
\begin{align*}
    \text{Reg}(T) \leq & \frac{2 \gamma}{2 - \omega} (6 m \log{T})^{\frac{\omega}{2}} T^{1 - \frac{\omega}{2}} +  \left( \frac{\pi^2}{3} + 1 \right) m \maxregret \ . \quad \qedhere
\end{align*}
\end{small}%
\end{proof}

Finally, we are ready to prove Theorem~\ref{thm:lipschitz_regret_bound}. 

\lipschitzRegretBound*
\begin{proof}
Eqs.~\ref{eqn:lipschitz_fixed_regret_bound} and~\ref{eqn:lipschitz_adaptive_regret_bound} can be proved by the same argument of the proof of Theorem~\ref{thm:fixed_discretization} and~\ref{thm:adaptive_discretization}. As shown in Eq.~\ref{eqn:adaptive_regret_bound_clear}, we get:
\begin{footnotesize}
\begin{align}
& \quad \text{Reg}(T) \nonumber \\
& \leq O \left(\frac{N^2}{L} \left( \frac{L^2 T}{N \log{T}} \right)^{\frac{2}{3}} \log{\left( \frac{T}{\log{T}} \right)} + N^2 L \left(\frac{L^2 T}{N \log{T}} \right)^{\frac{1}{3}} \right) \nonumber \\
& \leq O \left(L^{\frac{1}{3}} N T^{\frac{2}{3}} (\log{T})^{\frac{1}{3}} + L^{\frac{5}{3}} N^{\frac{5}{3}}  T^{\frac{1}{3}} (\log{T})^{-\frac{1}{3}} \right) \nonumber \\
& = O \left(L^{\frac{1}{3}} N T^{\frac{2}{3}} (\log{T})^{\frac{1}{3}} \right) \label{eqn:lipschitz_adaptive_regret_bound_clear}
\end{align}
\end{footnotesize}%
\end{proof}

\section{Experiment details}

The features associated with each target in Srepok Wildlife Sanctuary are distance to roads, major roads, rivers, streams, waterholes, park boundary, international boundary, patrol posts, and villages; forest cover; elevation and slope; conservation zone; and animal density of banteng, elephant, muntjac, pig. 

For LIZARD, CUCB, and zooming, we speed up the confidence radius from Eq.~\ref{eq:conf} as $r_t(i, j) = \sqrt{\frac{\epsilon}{2 n_t(i, j)}}$ with $\epsilon = 0.1$ to address the limited time horizon, which we empirically found to improve performance (no choice of $\epsilon$ affects the order of the algorithm's performance).
